\documentclass{article}

\usepackage[accepted]{aistats2022}

\usepackage[utf8]{inputenc} % allow utf-8 input
\usepackage[T1]{fontenc}    % use 8-bit T1 fonts
\usepackage{hyperref}       % hyperlinks
\usepackage{url}            % simple URL typesetting
\usepackage{booktabs}       % professional-quality tables
\usepackage{amsfonts}       % blackboard math symbols
\usepackage{nicefrac}       % compact symbols for 1/2, etc.
\usepackage{microtype}      % microtypography
\usepackage{xcolor}         % colors
\usepackage{algorithmic}
\usepackage{algorithm}
\usepackage{graphicx}
\usepackage{wrapfig}
\usepackage{framed}
\usepackage[numbers]{natbib}
\graphicspath{{figures/}}
\usepackage{subfigure}
\usepackage{amssymb}
\usepackage{amsthm}
\usepackage{tcolorbox}
\tcbuselibrary{theorems}
\newtheorem{proposition}{Proposition}

\newtheorem{definition}{Definition}

\newcommand{\newcontent}[1]{{#1}}
\usepackage{xparse}

\NewDocumentCommand{\codeword}{v}{%
\texttt{\textcolor{black}{#1}}%
}

\usepackage[textwidth=60pt]{todonotes}
\newcommand{\gab}[1]{}
\def\bbR{\mathbb{R}}
\def\bbE{\mathbb{E}}

\DeclareMathOperator*{\argmin}{arg\,min}
\DeclareMathOperator*{\Sp}{Sp}

\DeclareMathOperator{\Hess}{Hess}
\DeclareMathOperator{\Skew}{Skew}
\DeclareMathOperator{\Sym}{Sym}
\DeclareMathOperator{\grad}{Grad}

\DeclareMathOperator{\tr}{Tr}
\def\N{\mathcal{N}}

\def\Op{\mathcal{O}_p}
\def\stiefel{\mathcal{S}_{n ,p}}

\begin{document}



\twocolumn[

\aistatstitle{Fast and accurate optimization \\ on the orthogonal manifold without retraction}

\aistatsauthor{ Pierre Ablin \And Gabriel Peyré}

\aistatsaddress{CNRS, Département de mathématiques et applications\\ ENS, PSL University}]

\begin{abstract}
  We consider the problem of minimizing a function over the manifold of orthogonal matrices.
  The majority of algorithms for this problem compute a direction in the tangent space, and then use a retraction to move in that direction while staying on the manifold.
  Unfortunately, the numerical computation of retractions on the orthogonal manifold always involves some expensive linear algebra operation, such as matrix inversion, exponential or square-root.
  These operations quickly become expensive as the dimension of the matrices grows.
  To bypass this limitation, we propose the landing algorithm which does not use retractions.
  The algorithm is not constrained to stay on the manifold but its evolution is driven by a potential energy which progressively attracts it towards the manifold.
  One iteration of the landing algorithm only involves matrix multiplications, which makes it cheap compared to its retraction counterparts.
  We provide an analysis of the convergence of the algorithm, and demonstrate its promises on large-scale and deep learning problems, where it is faster and less prone to numerical errors than retraction-based methods.  
\end{abstract}

%%%%%%%%%%%%%%%%%%%%%%%%%%%%%%%%%%%%%%%%%%%%%%%%%%%%%%%%%%%%%%%%%%%%%%%%%%%%%%%%%%
%%
%%          INTRODUCTION
%%
%%%%%%%%%%%%%%%%%%%%%%%%%%%%%%%%%%%%%%%%%%%%%%%%%%%%%%%%%%%%%%%%%%%%%%%%%%%%%%%%%%

\section{Introduction}
\label{sec:intro}
We consider a differentiable function $f$ from $\bbR^{p\times p}$ to $\bbR$, and want to solve the problem
\begin{equation}
\label{eq:minproblem}
\min_{X\in\Op}f(X)\enspace,
\end{equation}
where $\Op$ is the \emph{Orthogonal manifold}, that is the set of matrices $X \in \bbR^{p\times p}$ such that $XX^{\top} = I_p$.
Problem~\eqref{eq:minproblem} appears in many practical applications, like principal component analysis, independent component analysis~\citep{comon1994independent,  nishimori1999learning, ablin2018faster}, procrustes problem~\citep{schonemann1966generalized}, and more recently in deep learning, where the \gab{isn't it two very different use of orthogonal matrices for deep nets (orthogonal filters and orthogonal fully connected maps) ?} weights of a layer are parametrized by an orthogonal matrix~\citep{arjovsky2016unitary, bansal2018can}.
This is a particular instance of minimization over a matrix Riemannian manifold, $\Op$~\citep{edelman1998geometry}.
Many standard Euclidean algorithms for function minimization have been adapted on Riemannian manifolds.
We can cite for instance gradient descent~\citep{absil2009optimization, zhang2016first}, second order quasi-Newton methods~\citep{absil2007trust, qi2010riemannian}, and stochastic methods~\citep{bonnabel2013stochastic} which are the workhorse for training deep neural networks.
More recently, several works propose to adapt accelerated methods in the Riemannian setting~\citep{zhang2018towards, tripuraneni2018averaging}.

\begin{figure}[t]
  \centering
\includegraphics[width=.6\columnwidth]{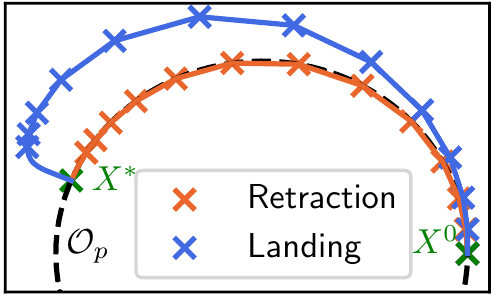}
  \caption{Trajectories of the landing algorithm and of a retraction gradient descent with $p=2$. The iterates are $2\times 2$ matrices, the $x$-axis corresponds to coefficient $(1, 1)$ of the matrices, and the $y$-axis to coefficient $(1, 2)$. The retraction method stays on $\Op$ (black dotted line) while the landing algorithm can deviate. Both methods start from $X_0$ and converge to the correct solution $X_*$. In higher dimension, the landing algorithm is much cheaper than the retraction method.}
\label{fig:illustration}
\end{figure}
All these methods are \emph{feasible}, i.e. generate a sequence of iterates $X_k$ where each iterate is in $\Op$.
Unlike what we assume in the first \gab{you mean the first part? or maybe ``for our approach?'' ? is it important here to insist on this?} sentence of the present article, they do not need the function $f$ to be defined outside $\Op$.
This comes with a computational drawback: in order to compute $X_{k+1}$ from $X_k$, one needs a way to move and stay on the manifold, called \emph{retraction}~\citep{absil2012projection}.
Unfortunately, retractions on $\Op$ are computationally expensive: they usually require a matrix inversion, square root, or exponential.
These operations are also generally slow on modern computing hardware such as GPU's.
Therefore, when the dimension $p$ is large, computing a retraction can become the computational bottleneck in the processing pipeline.

In this work, we propose the \textbf{landing algorithm}. It is an \emph{infeasible} method, which produces iterates $X_k$ that are not necessarily orthogonal, but which converge to a local minimum of $\eqref{eq:minproblem}$ as $k\to+\infty$.
The iterates get closer and closer to the manifold, and at the limit, \emph{land} on $\Op$.
The algorithm is illustrated in \autoref{fig:illustration} on a low dimensional problem.
The main advantage of the method is that the update rule is much simpler than a retraction since it involves only a few matrix multiplications.
As a result, our method can be much faster than standard feasible methods when $p$ is large.

Furthermore, retraction methods often suffer from an accumulation of numerical errors, which means that the iterates can get far from $\Op$ after many steps of the algorithm.
This effect is worsened by the low precision of floating point number that is customary in modern deep learning frameworks.
On the other hand, our method can only converge to matrices such that $\|XX^{\top} - I_p\| = 0$ to numerical precision.
Even though the proposed method is infeasible, it returns a solution that is closer to the manifold than most feasible methods in practice.
Infeasible methods on $\mathcal{O}_p$ have recently gained interest~\cite{xiao2020class, xiao2020exact, xiao2021penalty}. Closest to this work is~\cite{gao2019parallelizable}, which proposes a Lagragian based update. It is not robust to the choice of hyper-parameter, which makes it hard to use in practice (see \autoref{app:plam_vs_landing}).

The article is organized as follows: in Section~\ref{sec:preliminaries}, we recall some usual results about the geometry of $\Op$ and Riemannian optimization algorithms.
In Section~\ref{sec:landing_algo}, we introduce the landing algorithm and study global and local convergence. Some extensions are discussed.
Finally, experiments in Section \ref{sec:expe} show the benefit the landing algorithm over retraction methods in terms of computational efficiency and final distance to $\Op$.

\textbf{Notation}: $\Skew_p$ is the set of skew-symmetric matrices, $\Sym_p$ is the set of symmetric $p \times p$ matrices. The $\Skew$ of a matrix $M\in \bbR^{p\times p}$ is $\Skew(M) = \frac12(M - M^{\top})$, and the $\Sym$ is $\Sym(M) = \frac12(M + M^{\top})$. The Euclidean gradient of $f$ is $\nabla f$, the Riemannian gradient is $\grad f$. The norm is the Frobenius $\ell^2$ norm. The squared ``distance'' to the manifold is $\N(X) = \frac14\|XX^{\top} - I_p\|^2$.

We give sketches of proofs in the main text. Detailed proofs are in appendix.
\section{Preliminaries}
\label{sec:preliminaries}
We recall concepts about optimization on manifolds that will be useful in the rest of this article.
\subsection{Geometry of the orthogonal manifold}
\label{sec:geom_mani}
The orthogonal manifold is $\Op \triangleq \{X\in\bbR^{p\times p}|\enspace XX^{\top} = I_p\}$.
If $X(t)$ for $t\in [0, 1]$ is a differentiable curve on the manifold, differentiating the equation $X(t)X(t)^{\top}=I_p$ gives $\dot X(t)X(t)^{\top} + X(t) \dot X(t)^{\top}=0$, hence $\dot X(t) \in \mathcal{T}_{X(t)}$ where $\mathcal{T}_X$ is the \emph{tangent space} at $X$, given by
$
\mathcal{T}_X = \{\xi \in\bbR^{p\times p } |\enspace \xi X^{\top} + X \xi^{\top}=0\}
$.
We see that a matrix $\xi$ is in $\mathcal{T}_X$ if and only if for $A \in\Skew_p$ we have $\xi = AX$.
It then easily seen that the tangent space is a linear space of dimension $\frac{p(p-1)}{2}$.
The projection on the manifold $\mathcal{P}(X) \triangleq \argmin_{Y\in\Op}\|X - Y\|$ is $\mathcal{P}(X) = (XX^{\top})^{-\frac12}X$.
We now turn our attention to optimization on $\mathcal{O}_p$.
\subsection{Relative optimization on $\Op$ and extension to $\bbR^{p\times p}$}
Vectors in the tangent space at $X$ are of the form $AX$ with $A\in\Skew_p$.
The effect of small perturbations of $X$ in the direction $AX$ on $f$ leads to so-called \emph{relative} derivatives~\citep{cardoso1996equivariant}:
\begin{definition}
\label{def:relative_der}
For $X\in\bbR^{p\times p}$, the relative gradient $\psi(X)\in \Skew_p$ is defined with the Taylor expansion, for $A\in\Skew_p$: $
f(X + AX) = f(X) + \langle A, \psi(X)\rangle +  o(\|A\|).$
The relative Hessian $\mathcal{H}_X$ is the linear operator $\Skew_p \to \Skew_p$ such that $\psi(X + AX) = \psi(X) + \mathcal{H}_X(A) +  o(\|A\|).$
\end{definition}
These quantities are not defined only on $\Op$, but on the whole $\bbR^{p\times p}$, and can be computed easily from the Euclidean derivatives of $f$.
\begin{proposition}[Relative from Euclidean]
\label{prop:relat_eucl_ders}
Let $\nabla f(X) \in \bbR^{p\times p}$ the Euclidean gradient and $H_X:\bbR^{p\times p} \to \bbR^{p\times p}$ the Euclidean Hessian of $f$ at $X$. We have
$
\psi(X) = \Skew(\nabla f(X) X^{\top})
\quad\text{and} \quad
\mathcal{H}_X(A) = \Skew(H_X(AX)X^{\top} -\nabla f(X) X^{\top}A )
$
\end{proposition}
We can recover the \emph{Riemannian} gradient and Hessian of $f$ from the Relative derivatives:

\begin{proposition}[Riemannian from relative]
\label{prop:relat_riemannian_ders}
For $X\in\Op$, we have $\grad f(X) = \psi(X)X$, and for $A \in \Skew_p$, we have $\Hess f(X)(AX) = \mathcal{H}_X(A)X + \Skew(\psi(X)A)X $.
\end{proposition}
% \begin{proof}
% The Riemannian gradient of $f$ is such for $\xi\in\mathcal{T}_X$, it holds $f(X + \xi) = f(X) + \langle \grad f(X), \xi\rangle + o(\|\xi\|)$.
% %
% Letting $A= \xi X^{\top}\in \Skew_p$, we find $\langle \grad f(X), \xi\rangle = \langle \psi(X), A\rangle = \langle \psi(X)X, \xi\rangle$ which shows $\grad f(X) = \psi(X)X$.
% %
% Further, the Riemannian Hessian is given by $\Hess f(X)(\xi) = \Skew\left(H_X(\xi)X^{\top} - \Sym\left(\nabla f(X) X^{\top}\right)\xi X^{\top}\right)X$~\citep[Sec. 4.3]{absil2013extrinsic}, which concludes the proof.
% \end{proof}
%
Therefore, the critical points of $f$ on $\Op$, i.e. the points such that $\grad f(X)=0$, are exactly the points such that $\psi(X) =0$, and at those points, we have $\Hess f(X)(AX) = \mathcal{H}_X(A)X$: the Hessians are the same up to a remapping.

\subsection{Optimization on the orthogonal manifold with retractions}
\label{sec:optim_standard}
A simple method to solve Problem~\eqref{eq:minproblem} is the Riemannian gradient flow, which is the Ordinary Differential Equation (ODE) starting from $X_0\in \Op$
\begin{equation}
\label{eq:gradient_flow}
X(0) = X_0, \enspace \dot X(t) = -\grad f(X(t))\enspace.
\end{equation}
It is easily seen that the trajectory of the ODE stays in $\Op$, and that $f(X(t))$ decreases with $t$.
Further assumptions on $f$, like Polyak-Lojasiewicz inequalities~\citep{karimi2016linear, balashov2020gradient} or geodesic strong-convexity allow to prove the convergence of $X(t)$ to a minimizer as $t\to + \infty$.
If $f$ is Lipschitz then we have global convergence to a stationary point: $\lim\inf \|\grad f(X(t))\| = 0$.
In order to obtain a practical optimization algorithm, one should discretize the gradient flow.
Sadly, a naive Euler discretization, iterating $X_{k+1}= X_k - \eta \grad f(X_k)
$ with $\eta >0$ yields iterates which do not belong to the manifold, because the curvature is not considered.

This motivates the use of retractions.
A retraction $\mathcal{R}$ maps $(X, \xi)$ where $X\in\mathcal{O}_p$ and $\xi\in\mathcal{T}_X$ to a point  $\mathcal{R}(X, \xi)\in \Op$, and is such that $\mathcal{R}(X, \xi) = X + \xi + o(\|\xi\|)$.
Since the tangent space has such a simple structure, it is easier to describe a retraction with the mapping $\tilde{\mathcal{R}}(X, A)$, where $A\in \Skew_p$, such that $\tilde{\mathcal{R}}(X, A) = \mathcal{R}(X, AX)$.
\begin{table}[h!]
\begin{tabular}{l | c}
Name & Formula for $\tilde{\mathcal{R}}(X, A)$\\
\hline
Exponential & $\exp(A)X$\\
Projection & $\mathcal{P}(X + AX)$ \\
Cayley & $(I_p - \frac A2)^{-1}(I_p + \frac A2)X$ \\
QR & $\mathrm{QR}(X + AX)$
\end{tabular}
\caption{Popular retractions}\label{table:retractions}
\end{table}
Table~\ref{table:retractions} lists four popular retractions.
They all involve linear algebra operations on matrices like inversion, square root, or exponential.
There is no ``simpler" retraction:
\begin{proposition}[No polynomial retraction]
\label{prop:no_poly}
Fix $X\in\mathcal{O}_p$.
There is no polynomial $P(A)$ such that $\tilde{\mathcal{R}}(X, A)= P(A)$ is a retraction at $X$.
\end{proposition}
\begin{proof}
By contradiction, such polynomial must satisfy $P(A)P(A)^{\top} = I_p$. Thus, $PP^{\top}$ is of degree $0$, hence $P$ is of degree $0$, and $P$ is constant. Therefore, we cannot have $P(A) =X + AX + o(\|AX\|)$.
\end{proof}
Of course, in practice, most retractions are implemented using polynomial approximations (see e.g.~\cite{moler2003nineteen, li2020efficient}). The previous proposition simply shows that polynomials can only be approximations, and as a consequence, \emph{any} retraction must involve some linear algebra more complicated than matrix multiplication.

Riemannian gradient descent uses a retraction to stay on the manifold. It iterates
\begin{equation}
\label{eq:riemmanian_gd}
X_{k+1} =\mathcal{R}(X_k, -\eta \grad f(X_k))\enspace,
\end{equation}
where $\eta >0$ is a step-size.
Riemannian gradient descent is conveniently written with the relative gradient $\psi$ as $X_{k+1} = \tilde{\mathcal{R}}(X_k, -\eta\psi(X_k))$.
% We finish by noting that many works aim at improving on Riemannian gradient descent~\eqref{eq:riemmanian_gd}.
% %
% For instance, second order methods look for a better descent direction than $-\psi(X_k)$, by building an approximation of the Riemannian Hessian~\cite{absil2009optimization}.
We now present the landing algorithm, which \emph{does not require retractions}.
%%%%%%%%%%%%%%%%%%%%%%%%%%%%%%%%%%%%%%%%%%%%%%%%%%%%%%%%%%%%%%%%%%%%%%%%%%%%%%%%%%
%%
%%          METHODS
%%
%%%%%%%%%%%%%%%%%%%%%%%%%%%%%%%%%%%%%%%%%%%%%%%%%%%%%%%%%%%%%%%%%%%%%%%%%%%%%%%%%%

\section{The landing algorithm}
\label{sec:landing_algo}
In the following, we use the function $\N(X) \triangleq \frac14\|XX^{\top} - I_p\|^2$.
This function is minimized if and only if $X\in \Op$.
A simple way to build an algorithm that converges to $\Op$ consists in following $-\nabla \N(X)$, which leads in the continuous setting to Oja's flow $\dot X = - \nabla \N(X)$~\citep{oja1982simplified, yan1994global} and in the discrete setting to Potter's algorithm $X_{k+1} = X_k  - \eta \nabla \N(X_k)$~\citep{cardoso1996equivariant}.
Note that the Euclidean gradient has the simple formula $\nabla \N(X) = (XX^{\top} - I_p)X$, and that it is always orthogonal to the Riemannian gradient of $f$, since $\nabla \N(X)$ is written as $SX$ with $S\in \Sym_p$, while $\grad f(X)$ is written as $AX$ with $A\in\Skew_p$.

The landing algorithm combines the previous orthogonalizing method with the minimization of $f$.
\begin{figure}[h!]
\centering
\begin{framed}\raggedleft
\vspace{-.5em}
\includegraphics[width=\linewidth]{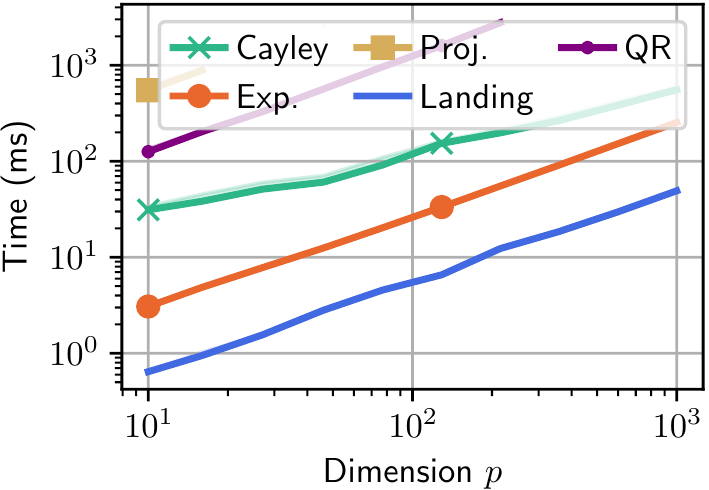}
\vspace{-1.5em}
\caption{Time required to compute $500$ retractions $\tilde{\mathcal{R}}(X, A)$ when $A$ and $X$ are of size $p\times p$, on a GPU.}
\vspace{-.5em}
\end{framed}
\vspace{1.2em}
\label{fig:comp_speed}
\end{figure}
We define the \textbf{landing field} as the mapping $\bbR^{p \times p} \to \bbR^{p\times p}$:
\begin{equation}
\label{eq:landing_field}
\boxed{
\Lambda(X) \triangleq \psi(X)X + \lambda\nabla \N(X)
}\enspace,
\end{equation}
where $\lambda >0$ is a fixed parameter.
This allows us to define the \textbf{landing algorithm}, which iterates:
\vspace{-.5em}
\begin{equation}
  \vspace{-.5em}
\label{eq:landing_algo}
X_{k+1} = X_k - \eta^k \Lambda(X_k)\enspace,
\end{equation}
with $\eta^k > 0$ a sequence of step-sizes.
Its continuous counterpart is the \textbf{landing flow}:
\vspace{-.5em}
\begin{equation}
  \vspace{-.5em}
\label{eq:landing_flow}
\dot X(t)=- \Lambda(X(t)).
\end{equation}
We stress that the field $\Lambda$ is not the Riemannian gradient nor the Euclidean gradient of a function (its Jacobian is not symmetric).
In particular, the landing flow does \emph{not} have the same trajectory as the Euclidean gradient flow associated to the function $f(X) + \lambda \mathcal{N}(X)$.

Before we move on to the analysis of the algorithm, we can already see that one iteration of the landing algorithm only involves some matrix multiplications instead of expensive linear algebra.
In \autoref{fig:comp_speed}, we show the cost of computing on a GPU one iteration of the Riemannian gradient descent using the standard retractions, and the cost of computing one iteration of the landing algorithm as $p$ grows. The proposed method is about $8$ times faster than retraction methods.

\textbf{Comparison to penalty methods}
An idea to get an approximation of problem~\eqref{eq:minproblem} is to minimize, without constraint, the penalized function $g(X)  = f(X) + \lambda \mathcal{N}(X)$. This is conceptually simpler than Riemannian optimization, and can be implemented very easily. However, the main drawback of this method is that the solution will not in general be feasible. Furthermore, to implement this method, we need to compute the gradient of $g$ given by $\nabla f(X) + \lambda (XX^\top -I_p)X$. On top of computing the gradient, we see that it involves $2$ matrix multiplications. By comparison, computing the landing field $\Lambda$ requires $3$ matrix multiplications. Hence, computing the landing field is only $50\%$ more costly than the gradient of $g$, and as we will see, it provides us with a feasible solution.

\textbf{Computational cost of Riemannian gradient descent}
Riemannian gradient descent on $\Op$ first computes the descent direction, using the Euclidean gradient of $f$, and then computes the next iterate using a retraction. Depending on the problem, the main computational bottleneck may come from either of the two steps.
For instance, when training a Recurrent Neural Network (RNN) with orthogonal weights~\citep{arjovsky2016unitary, helfrich2018orthogonal, lezcano2019trivializations}, there is usually only one orthogonal matrix used in a large computational graph. Here, the cost of computing the Euclidean gradient with backpropagation is much higher than the cost of computing a retraction. Consequently, using the landing algorithm in this setting will only slightly reduce the cost of computations.
On the other hand, it is common to impose an orthogonality constraint on multilayer perceptron or convolutional neural networks~\citep{rodriguez2016regularizing, bansal2018can}. It has been reported that this constraint allows faster training and better generalization. In this case, there are many orthogonal matrices, and the main bottleneck in training can be computing the retraction. Therefore, as we will see in the experiments, it is interesting to use the landing algorithm in this setting.

We now turn to a theoretical analysis of the method, and begin by a study of the critical points.
\begin{proposition}[Critical points of $\Lambda$]
\label{prop:stationnary}
Let $X\in \bbR^{p\times p}$ invertible. We have $\Lambda(X) = 0$ if and only if $X\in\Op$ and $\psi(X) = 0$.
\end{proposition}
\begin{proof}
If $\Lambda(X) = 0$, we have $\psi(X) + \lambda (XX^{\top} - I_p) = 0$.
Looking at the symmetric part, we obtain $XX^{\top} = I_p$, i.e. $X\in \Op$.
Looking at the skew-symmetric part, we obtain $\psi(X) =0$.
Conversely, if $\psi(X) = 0$ and $X\in\Op$, we have $\Lambda(X) =0$.
\end{proof}
This result shows that the stationary points of the landing algorithm are the stationary points of the original problem~\eqref{eq:minproblem}. It holds regardless of the value of the hyper-parameters $\lambda$.
\newcontent{We also stress that the invertibility condition on $X$ is not a problem in practice: as we will see in the next section, the iterates stay close enough to $\mathcal{O}_p$ so that they are bounded away from the singular matrices set: the stationary points for the landing algorithm/flow are the stationary points of $f$.}
\subsection{Orthogonalization property}
\label{sec:orthogonalization}
We start by showing that the landing flow~\eqref{eq:landing_flow} is \newcontent{well defined and} \emph{orthogonalizing}: the flow converges to the orthogonal manifold regardless of initialization.
\begin{proposition}[Convergence of the flow to $\Op$]
\label{prop:ortho}
\newcontent{There is a solution $X(t)$ of the landing flow~\eqref{eq:landing_flow} defined for all $t\geq 0$.} Then, $\N(X(t))$ decreases, and denoting $N_0 \triangleq \N(X_0)$, we have $\N(X(t)) \leq e^{-\lambda t} \left(\frac{N_0}{(\sqrt{N_0} - 1)^2}\right)$.
\end{proposition}
\begin{proof}
Let $n(t) = \N(X(t))$. We find $n'(t) = \langle \dot X(t), \nabla \N(X(t))\rangle$. Then, we have for all $X$, $\langle \psi(X)X, \nabla \N(X)\rangle = \langle \psi(X), (XX^{\top} - I_p)XX^{\top}\rangle$. The matrix on the left is skew-symmetric, the matrix on the right is symmetric, hence this scalar product cancels. Therefore, we get $n'(t) = -\lambda \|\nabla \N(X(t))\|^2$.
This shows that $n(t)$ decreases. \newcontent{This proves the existence of a solution for all times, by a standard Lyapunov argument.} Next, we use the inequality $\|(XX^{\top} - I_p)X\|^2 \geq \N(X) - \N(X)^{\frac32}$ which gives us $n'(t) \leq - \lambda(n(t) - n(t)^{\frac32})$. This inequality is then integrated to obtain the result.
\end{proof}

This shows that the landing flow produces a trajectory that \emph{lands} on the manifold: the distance to the manifold decreases at a linear rate to $0$.
If the landing flow starts on the manifold ($X_0 \in \Op$), then $\N(X(t)) = 0$ for all $t\geq 0$, i.e. the flow stays on the manifold, and is equal to the Riemannian gradient flow~\eqref{eq:gradient_flow}.

\textbf{Safe rule for the discrete algorithm}
\newcontent{The convergence of the landing \emph{algorithm} towards $\Op$ is more complicated to study}.
For instance if $X_0\in\Op$, then $X_1 = X_0 - \eta\psi(X_0)X_0$ is not orthogonal unless $\psi = 0$.
Therefore, there is no hope that $\N(X_k)$ is a decreasing sequence.
Instead, we set $\varepsilon > 0$, and get a criterion on the step-size which ensures $\N(X_k)\leq \varepsilon$ for all $k$.
\begin{proposition}[Safe step-size interval]
\label{prop:safe_rule}
Assume that $X_k$ is such that $d \triangleq \N(X_k)\leq \varepsilon$. Let $a \triangleq \|\psi(X_k)\|$ and
$
\eta^*(a, d) \triangleq \frac{\sqrt{\alpha^2 + 4 \beta(\varepsilon - d)} +  \alpha}{2\beta}, \enspace
$
where $\alpha \triangleq 2\lambda d -2 ad -2 \lambda d^2$ and $\beta \triangleq a^2 + \lambda^2d^3 + 2\lambda ad^2 + a^2d$.
Then if $\eta \in[0, \eta^*(a, d)]$, we have $\N(X_{k+1})\leq \varepsilon$.
\end{proposition}

As a consequence, if the algorithm starts from $X_0 \in \Op$ and $\eta$ is in the safe interval for each step, the iterates  all verify $\N(X_k) \leq \varepsilon$.
It is worth mentioning that while the above formula is complicated, it is only a matter of computing a scalar function given $\|\psi(X_k)\|$ and $\N(X_k)$, so computing $\eta^*(a, d)$ is negligible in front of the other computations.
In practice, we provide a sequence of target step-size $\eta_k$ to the algorithm, and at each iteration, we compute $\eta^*$. We then use $\min(\eta_k, \eta^*)$ as the step-size.
\begin{figure}[t!]
  \begin{algorithm}[H]
  \caption{Landing algorithm with safe step-size}
  \label{algo:landing}
    \begin{algorithmic}
       \STATE {\bfseries Input :} Initial point $X_0 \in \Op$, step-size sequence $\eta_k$, number of iterations $N$.
       \FOR{$k=1$ {\bfseries to} $N$}
       \STATE Compute $\eta^*$ (\autoref{prop:safe_rule})
       \STATE Set $\eta_k = \min(\eta^*, \eta_k)$
       \STATE Update $X_{k+1} = X_k -\eta_k \Lambda(X_k)$
        \ENDFOR
         \STATE \textbf{Return : }$X_N$.
        \end{algorithmic}
  \end{algorithm}
\end{figure}
Importantly, we see that when $a=0$, the safe step-size is of the order $\eta^* \simeq \frac{4}{\lambda d^2}\leq \frac{4}{\lambda \varepsilon^2}$, which is large when $\varepsilon$ is small.\newcontent{When $d= 0$, we have $\eta^*=\frac{\sqrt{\varepsilon}}{a}$, which is reminiscent of the baseline step-size (inverse of Lipschitz constant of the problem). In practice, we take $\varepsilon=\frac12$, which ensures that the safe-step size is not small in the two previous settings.}
This safe rule therefore does not restrict much the choice of step-size, which is also observed in practice.
This gives us a safe landing algorithm, described in \autoref{algo:landing}.
We stress that the condition $\mathcal{N}(X)\leq \varepsilon$, which is imposed using this safe-step technique, guarantees that $X$ is invertible as soon as $\varepsilon < 1$, since for any singular $X$ we have $\mathcal{N}(X) \geq 1$.

\newcontent{Convergence of $N(X_k)$ to $0$ depends on the convergence of $\psi(X_k)$ to $0$, which requires global convergence results, presented later in Section~\ref{subsec:global}.}
% We rather view $\psi(X_k)$ as a perturbation quantity that goes to $0$.
% %
% We consider two settings: one where convergence is linear, e.g. $\|\psi(X_k)\|\simeq \omega^k$ for some $\omega < 1$, and one where convergence is sub-linear, e.g. $\|\psi(X_k)\|\simeq k^{\alpha}$ for some $\alpha < 0$.

% \begin{proposition}[Linear convergence of the algorithm to $\Op$]
% \label{prop:lin_conv}
% %
% Assume that the sequence $a_k = \|\psi(X_k)\|$ satisfies $a_k = O(\omega^k)$ for some $\omega \in (0, 1)$, and that the step-size is fixed to $\eta \leq \frac1{2\lambda}$. For any $\delta >0$, we have
% $
% \N(X_k) = O\left((\max(1 - 2 \eta \lambda, \omega) +  \delta)^k \right)
% $.
% \end{proposition}
% \begin{proposition}[Sublinear convergence of the algorithm to $\Op$]
% \label{prop:sublin_conv}
% %
% Assume that the sequence $a_k = \|\psi(X_k)\|$ satisfies $a_k = O(k^{\alpha})$ for some $\alpha < 0$, and that the step-size is fixed to $\eta \leq \frac1{2\lambda}$. For any $\delta >0$, we have
% $
% \N(X_k)= O\left(k^{\alpha + \delta} \right).
% $
% \end{proposition}
% These propositions show that $d_k$ goes to $0$, at the slowest rate between the rate of convergence of the minimization of $f$, and the rate of convergence of the orthogonalization algorithm when $\psi = 0$.
% %
% This behavior is observed in practice.

\paragraph{Stochastic method and distance to $\Op$ in the small gradient regime}
When $f$ has a sum structure, $f(X) = \sum_{i=1}^n f_i(X)$, it is possible to use stochastic gradient descent, which takes a step in the opposite direction of the gradient of one of the $f_i$ instead of $f$. Such method is easily adapted to the Riemannian setting, by taking Riemannian stochastic gradients and using retractions or the landing algorithm.
Defining $\psi_i$ as the relative gradient of the function $f_i$, the stochastic landing algorithm samples $i_k$ at random between $1$ and $n$, and then does a step $X_{k+1} = X_k - \eta_k \left(\psi_{i_k}(X_k) + \lambda (X_kX_k^{\top} - I_p)\right)X_k$.
%
% In general, stochastic methods cannot converge when the step-size $\eta_k$ is fixed. Here, even if $X_0 = X_*\in \argmin_{\Op} f(X)$, we have $X_1 = X_0 - \eta_0\psi_{i_0}(X_0)X_0$. Since in general $\psi_{i_0}(X_0) \neq 0$, we have $X_1 \neq X_*$: the algorithm moves away from critical points.
%

We  now detail an informal computation to control the distance of the iterates to $\Op$ when the gradients $\psi_i(X_k)$ are small.
Denoting $\Delta_k \triangleq X_kX_k^{\top} -I_p$, and neglecting high order terms in $\psi$ and $\Delta_k$, one has the approximate relationship
$
  \Delta_{k+1}\simeq (1 - 2\eta_k\lambda)\Delta_k - \eta_k^2\left(\psi_i(X_k)\right)^2.
$
Assuming that the gradients $\psi_i(X_k)$ are independent from $\Delta_k$ and have an average norm $a$, we find
$\bbE[\|\Delta_{k+1}\|^2] =  (1 - 2\eta_k\lambda)^2\bbE[\|\Delta_{k}\|^2] + \eta_k^4a^4.
$
If the step-sizes $\eta_k$ are fixed to $\eta > 0$, the above equation indicates that $\bbE[\|\Delta_{k}\|^2]$ converges to a limit value given by $
\left(\bbE[\|\Delta_{k}\|^2]\right)^{1/2} \to \delta_* \triangleq \frac{\eta a^2}{2 \lambda}.$
The above reasoning is informal and there are many approximations. However, we find that $\delta^*$ is close to the distance to the manifold $\Op$ observed in practice.
\subsection{Local convergence}
In this section, we assume that the iterates are close to a local minimum of~\eqref{eq:minproblem}, and study its stability.
We let $X_*\in \Op$ such that $\psi(X_*) = 0$ and $\mathcal{H}_{X^*}$ is positive, and study its stability. We let $\mu_{\min} >0 $ the smallest eigenvalue of $\mathcal{H}_{X^*}$.
\begin{proposition}[Local convergence, landing flow]
\label{prop:local_landing}
For any $\delta > 0$, there exists $\epsilon> 0$ such that if $\|X_0 - X_*\|\leq \epsilon$, the landing flow starting from $X_0$ verifies $\|X(t) - X_*\| = O\left(\exp(-(\min(\mu_{\min}, \lambda) + \delta)t)\right)$.
\end{proposition}
Therefore, if $\lambda \geq \mu_{\min}$, we get the same local convergence speed for the landing flow and the Riemannian gradient flow. We obtain a similar result in the discrete case, using a Lipschitz assumption.

\begin{proposition}[Local convergence, landing algorithm]
\label{prop:local_convergence}
Assume that $\Lambda$ is $L-$Lipschitz.
Then for any $\delta > 0$ there exists $\epsilon> 0$ such that if $\|X_0 - X_*\|\leq \epsilon$, the landing algorithm starting from $X_0$ with constant step $\eta \leq \frac1L$ verifies $\|X_k - X_*\| = O\left((1 - \frac{\min(\mu_{\min}, \lambda)}{L} + \delta)^k\right)$.
\end{proposition}
These two results follow from the expression of the Jacobian of the field $\Lambda$ at $X_*$.
Once again, when $\lambda \geq \mu_{\min}$, we get the same rate as Riemannian gradient descent~\citep{zhang2016first}.

\textbf{Hyper-parameter trade-off}
The hyper-parameter $\lambda$ plays a key role in the convergence results. \autoref{prop:local_convergence} suggests that $\lambda$ should be chosen to maximize $\frac{\min(\mu_{\min}, \lambda)}{L}$. Since $L$ is the Lipschitz constant of $\Lambda$, we have an upper bound of the form $L\leq l_1 + \lambda l_2$ with $l_1, l_2$ the respective Lipschitz constants of $\psi(X)X$ and $\nabla\mathcal{N}(X)$.
Then,  $\frac{\min(\mu_{\min}, \lambda)}{L}$ is maximized for $\lambda = \mu_{\min}$: this is in theory the best value of $\lambda$ to get fast local convergence.
However, this constant is usually intractable.
In the experiments, we take $\lambda=1$, which in practice gives satisfying results.

\subsection{Global convergence}
\label{subsec:global}
We now give a global convergence result for the landing flow:
\begin{proposition}[Global convergence, continuous case]
\label{prop:global_conv}
Let $T\geq 0$. We assume that for all $t \leq T$, $\|\Sym(\nabla f(X(t))X(t)^{\top})\|\leq K$, and we let $f^* = \min f$. We have
$$
\inf_{t\leq T}\|\psi(X_t)\|\leq\frac{1}{\sqrt{T}}\left(f(X_0) - f^* +  2K\frac{\sqrt{N_0}}{\sqrt{N_0} - 1}\right)^{\frac12}.
$$
\end{proposition}
This shows global convergence of the flow at the usual rate $1/\sqrt{T}$. This result is analogous to the one one would get following the Riemannian gradient flow on the manifold (e.g.~\cite[Prop. 4.6]{boumal2020introduction}).

In the discrete case, we show that the landing algorithm with constant step-size $\eta$ produces iterates that get at a distance of the order $\eta$ to the stationary points.
\begin{proposition}[Global convergence, discrete + fixed step-size case]
\label{prop:discrete_global_conv}
Let $X_k$ the sequence of iterates of the landing algorithm with step-size $\eta$, starting from $X_0\in\Op$. There exists constants $\delta, C_1, C_2 > 0$ (given in Appendix) such that when $\eta\leq \delta$, it holds $\mathcal{N}(X_k)\leq \eta \cdot C_1$ and $\inf_{k\geq 0} \|\psi(X_k)\|^2\leq \sqrt{\eta}\cdot C_2$.
\end{proposition}

\newcontent{We have not been able to show stronger convergence results in the fixed step-size regime. Based on empirical evidence, we conjecture that for $\eta$ small enough we have $\lim_{k \rightarrow + \infty} \|\psi(X_k)\|^2 =0$ and $\lim_{k\to +\infty} \mathcal{N}(X_k) = 0$. The following proposition shows convergence of the algorithm with decreasing step-size:}
\begin{proposition}[Global convergence, discrete + decreasing step-size case]
  \label{prop:discrete_diminish}
  \newcontent{Let $X_k$ the sequence of iterates of the landing algorithm with step-size $\eta_k = k^{-\alpha}$ with $\alpha\in (\frac12, 1)$, starting from $X_0\in\Op$. Then, $\mathcal{N}(X_k)=\mathcal{O}(k^{-\alpha})$ and $\inf_{k\geq 0} \|\psi(X_k)\|^2 = O(k^{-\min(\frac\alpha 2, 1 - \alpha)})$.}
\end{proposition}

\newcontent{This proposition shows convergence of the landing algorithms: the iterates land on the manifold since $\mathcal{N}(X_k)$ goes to $0$, and they go towards stationary points of $f$ since $\inf_{k\geq 0} \|\psi(X_k)\|^2$ goes to $0$. The best rate of convergence is obtained for $\alpha = \frac23$ and we find $\inf_{k\geq 0} \|\psi(X_k)\|^2 = O(k^{-\frac13})$. In contrast, Riemannian gradient descent achieves a rate of $O(k^{-1})$.}

\subsection{Acceleration with momentum}
It is straightforward to derive a momentum version of the landing algorithm by accumulating the relative gradients.
Starting from the initial speed $A_0 = 0$, for a momentum term $\gamma \in [0, 1]$, the landing algorithm with momentum iterates
\begin{equation}
  \begin{cases}
    A_{k+1} = (1 - \gamma)A_k + \gamma \psi(X_k) \\
    X_{k+1} = X_k - \eta_k(A_k X_k + \lambda \nabla \mathcal{N}(X_k)).
  \end{cases}
\label{eq:acc_mom}
\end{equation}
In Eq.~\eqref{eq:acc_mom}, the relative gradient can be replaced by a stochastic estimate. This leads to significant acceleration in the deep learning experiments.
The corresponding second order ODE is
\begin{equation}
  \begin{cases}
    \dot{A}(t) = -A(t) +  \psi(X(t)) \\
    \dot{X}(t) = - \left(A(t) + \lambda (X(t)X(t)^\top - I_p)\right) X(t)
  \end{cases}
\end{equation}
It is readily seen that $A(t)$ is skew-symmetric for all $t$, and therefore that we get the same convergence result as Prop.~\ref{prop:ortho}. Classical arguments with the Lyapunov function $f(X(t)) + \frac12 \|A(t)\|^2$ also provide global convergence : $\lim\inf \|\psi(X(t))\| = 0$ (See Appendix): the analysis in the continuous case is almost as straightforward as with no momentum.

\subsection{A landing field for other manifolds ?}
\label{sec:modifications}
The landing field can in principle be extended to (sub-)manifolds $\mathcal{M}$ of $\bbR^d$ that are \emph{orientable}(see e.g. \cite[Chapter 3]{boumal2020introduction} and \cite[Prop 15.23]{lee2013smooth}).
Indeed, one can derive a field $G(x)$ and a potential $\mathcal{N}(x)$ such that when $x\in \mathcal{M}$, $G(x) = \grad f(x)$, such that $\nabla \mathcal{N}(x)$ is $0$ if and only if $x\in \mathcal{M}$, and such that $G(x)$ and $\nabla \mathcal{N}(x)$ are always orthogonal. These properties are sufficient to obtain \autoref{prop:stationnary}.
However, these maps might not be tractable, while on $\Op$ their expressions are simple and cheap to compute.
%
% We describe two other classes of manifold on which one can also easily define a landing method.

% \textbf{Sphere manifold} The sphere is $\mathcal{S} = \{x\in\bbR^d|\enspace \|x\|= 1\}$. We can define the ``distance'' as $\N(X) = \frac14(\|x\|^2 -1)^2$, and
% $G(x)= \nabla f(x) - \frac{\langle \nabla f(x) , x\rangle}{\|x\|^2}x$.
% %
% The  landing field is then $\Lambda(x) = \nabla f(x) -\frac{\langle \nabla f(x) , x\rangle}{\|x\|^2}x + \lambda (\|x\|^2-1)x$.
% %
% However, there is no computational advantage of using the corresponding landing algorithm, because there are some cheap retractions on the sphere (for instance, computing the orthogonal projection only requires a scalar product).
% 
\textbf{Stiefel manifold} The Stiefel manifold $\stiefel$ is the set of \emph{rectangular} matrices $X\in\bbR^{n\times p}$ with $n > p$ such that $X^{\top}X = I_p$.
The Riemannian gradient of $f$ is once again given by the formula $\grad f(X) = \psi(X)X$, with $\psi(X) =\Skew(\nabla f(X)X^{\top})$.
Here, $\psi(X)$ is a large $n \times n$ matrix,\newcontent{but only $\psi(X)X$ appears in the formula which can be computed at a $O(n\times p^2)$ cost.}
The distance function becomes $\mathcal{N}(X) = \|X^{\top}X - I_p\|^2$, and the landing field can then be defined as $\Lambda(X) = \psi(X)X + \lambda  \nabla \mathcal{N}(X)$.
We obtain the equivalent of \autoref{prop:stationnary} and \autoref{prop:ortho}: the points such that $\Lambda(X)= 0$ are exactly those for which $X\in\stiefel$ and $\grad f(X) =0$, and we get a similar orthogonalization property of the flow.
We finish by stressing that some ``fast'' retractions are available for $\stiefel$ when $p$ is much smaller than $n$: Cayley retraction can be computed by inverting a small $2p \times 2p$ matrix. In this setting, the landing flow might not be much faster than this retraction.

\subsection{Numerical errors}
\label{sec:num_errors}

An advantage of our method is that it is robust to numerical errors.
Indeed, at convergence, the landing flow goes to a point $X$ such that $\|\Lambda(X)\|^2\leq \delta_{\mathrm{num}}$, where $\delta_{\mathrm{num}}$ is a small constant that depends on the floating point precision. Therefore, $\|XX^{\top} - I_p\|^2\leq \delta_{\mathrm{num}}$:
at the limit, the orthogonalization error is of the order of the floating point precision.
This is observed in practice.

On the contrary, consider for instance the exponential retraction.
Starting from $X_0 \in \Op$, it iterates $X_{k} = \exp(A_{k})X_{k-1}$, where $A_k$ is a skew-symmetric matrix.
For simplicity, assume that $X_0 = I_p$.  The iterate $X_k$ can be compactly rewritten as $X_k = \prod_{i=1}^k\exp(A_i)$. Hence, if the $\exp(A_i)$ are not perfectly orthogonal because of numerical errors, $X_k$ can get further and further from orthogonality as $k$ increases. Therefore, the landing algorithm, while it is a non-feasible method, gives a solution that is more orthogonal than methods using the exponential or the Cayley retraction.
% Since $A_k$ is skew-symmetric, all the  $\exp(A_i)$ are orthogonal, so we indeed have $X_{k}\in \Op$.
% %
% But numerical computation of the matrix exponential comes with some numerical error, which means that $\exp(A_i)$ is not perfectly orthogonal in practice.
% %
% To model the deviation to orthogonality, we write $\exp(A_i) = U_i + E_i$, where $U_i\in \Op$, and $E_i$ is a matrix of small norm. We let $P_i = U_i\times\dots\times U_1$.
% %
% A first order expansion then gives $X_k \simeq P_k + P_k\sum_{i=1}^kP_i^{\top}E_iP_i$:
% %
% the orthogonalization error increases with the number of iterations $k$.
% %
% The same reasoning applies to Cayley retraction.
%

%%%%%%%%%%%%%%%%%%%%%%%%%%%%%%%%%%%%%%%%%%%%%%%%%%%%%%%%%%%%%%%%%%%%%%%%%%%%%%%%%%
%%
%%          EXPERIMENTS
%%
%%%%%%%%%%%%%%%%%%%%%%%%%%%%%%%%%%%%%%%%%%%%%%%%%%%%%%%%%%%%%%%%%%%%%%%%%%%%%%%%%%

\section{Experiments}
\label{sec:expe}

We conclude by showcasing the usefulness of the landing algorithm on an array of optimization problems.
The deep learning experiments are run on a single Tesla V100 GPU with Pytorch~\citep{paszke2019pytorch}, while the other experiments are run on a small laptop CPU and Numpy~\citep{harris2020array}.
The code for the landing flow as a Pytorch \codeword{Optimizer} and to reproduce the experiments is available at \url{https://github.com/pierreablin/landing}.
In all experiments, we use the safe rule for the step-size described in \autoref{prop:safe_rule}, with $\varepsilon=0.5$, and set $\lambda=1$.

\begin{figure}[h!]
\centering
\includegraphics[width=.99\linewidth]{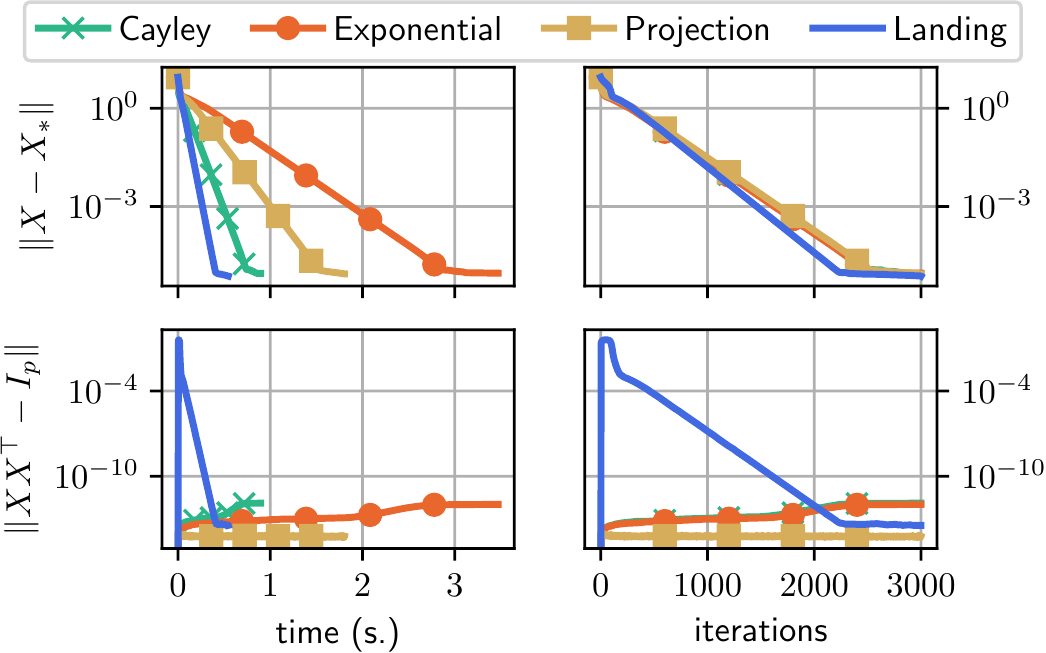}
  \caption{Gradient descent on the orthogonal procrustes problem using different retraction methods and the landing algorithm. Top: distance to the optimum. Bottom: distance to the orthogonal manifold. Left: w.r.t. time. Right: w.r.t. iterations.}
  \label{fig:procrustes}
  \vspace{-1em}
\end{figure}
\textbf{Orthogonal procrustes}
We let $A, B$ two $p\times p$ matrices, and define the Procrustes cost function as $f(X) = \|XA-B\|^2$ where $X\in\Op$.
We set $p=40$. We generate $A$ and $B$ two random matrices with i.i.d. normal entries.
We apply the different algorithms with a fixed step-size $\eta=0.1$.
We record the distance to the solution $X_*$, and the orthogonalization error $\|XX^{\top} - I_p\|$. \autoref{fig:procrustes} displays the results.
Looking at the distance to the optimum, all methods are similar in term of iterations.
Since one iteration of the landing algorithm is cheaper, we get an overall faster method. Looking at the distance to the manifold, the landing algorithm starts by moving away from the manifold, but in the end lands on the manifold. The exponential and Cayley retractions suffer from numerical errors, and end up being further from the manifold than the landing algorithm.

\begin{figure}[h!]
\centering
\includegraphics[width=.85\linewidth]{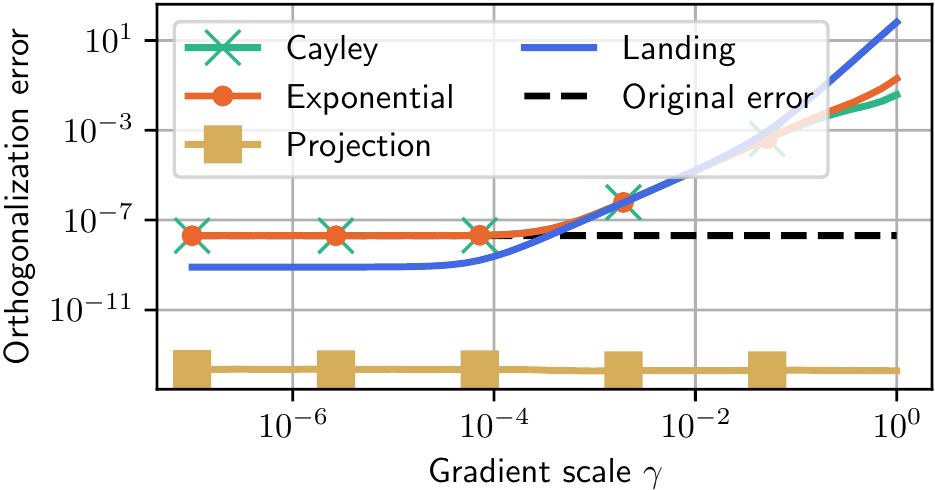}
  \caption{Orthogonality error after one step of each algorithm, starting from a matrix that is close to, but not in $\Op$.}
  \label{fig:ortho_propo}
  \vspace{-1em}
\end{figure}
\textbf{Orthogonalization property} We illustrate the orthogonalization property of the landing algorithm.
We take $p=100$, $X_0 = I_p$, and add a small error $X = X_0 + E$, where $E$ has i.i.d. entries of law $\mathcal{N}(0, \sigma^2)$, with $\sigma = 10^{-4}$. Therefore, $X$ is close to, but not perfectly in $\Op$.
Then, we generate a random `gradient', that is a random matrix $A\in \Skew_p$ of i.i.d. entries where $A_{ij} \simeq\mathcal{N}(0, \gamma^2)$ for $i>j$, where $\gamma$ controls the scale of the `gradient' $A$. Setting $\gamma$ small emulates an optimization problem closed from being solved, and $\gamma$ high an optimization problem far from being solved.
We then take a step in the direction of $A$, with step-size $\eta=.3$. For the landing algorithm, we set $\lambda=1$. In other words, we take the output as $X_{out} = \tilde{\mathcal{R}}(X, \eta A)$ for the retraction algorithms, and $X_{out} = X - \eta\left(A + XX^{\top} - I_p\right)X$ for the landing algorithm.
We then record the orthogonality error of the output, $\|X_{out}X_{out}^{\top} - I_p\|$.
For each $\gamma$, we repeat the experiment $50$ times with different random seeds.
\autoref{fig:ortho_propo} shows the average orthogonality error of the different algorithms as a function of the gradient scale $\gamma$.
The projection retraction yields a very small orthogonalization error.
The exponential and Cayley retraction do not have this correcting effect: the orthogonalization error stays the same if $\gamma$ is small. When $\gamma$ gets large enough, they increase the orthogonalization error.
Finally, the landing algorithm has a hybrid behavior. When the $\gamma$ is small, it acts mainly as a cheap projection algorithm, which means that the orthogonalization error decreases. Here, one iteration reduces the error by a factor $\simeq10$, so after a few iterations the algorithm would reach numerical precision.
%
% Then, as $\gamma$ increases, we observe the same behavior as the other algorithms: the error increases. Finally, when the gradient scale is large enough, since the algorithm does not stay on the manifold, we see that the error gets bigger than all the other algorithms.
%
This illustrates a ``self-correcting'' behavior of the landing algorithm: unlike the exponential and Cayley retractions, it can decrease the orthogonalization error.

\textbf{Deep learning}
We now turn to applications in deep learning, where the function $f$ involves a neural network.
In this part, we discard the projection retraction, which is orders of magnitude more costly to compute than other retractions on a GPU (see \autoref{fig:comp_speed}).

%
% \begin{figure}[ht]
%     % \hfill
%     \begin{minipage}[c]{0.60\linewidth}
%         \includegraphics[width=\textwidth]{simple_network_final.pdf}
%     \end{minipage}\hfill
%     \begin{minipage}[c]{0.39\linewidth}
%         \vspace{-1em}
%     \caption{Training a fully connected neural network with orthogonal weights. Left: average orthogonalization error of all weights as training progresses. Right: test loss.}\label{fig:simple_network}
%     \end{minipage}
%     \vspace{-1em}
% \end{figure}

% \begin{figure}
% \centering
% \includegraphics[width=.7\columnwidth]{simple_network_final.pdf}
%   \caption{Training a fully connected neural network with orthogonal weights using Riemannian stochastic gradient descent. We use the different retractions mentioned in the text, and the proposed landing algorithm. We exclude the projection retraction which is much more costly. Top: average orthogonalization error of all weights, as training progresses. Bottom: test loss.}
%
% \end{figure}
% %
\textbf{Distillation}
We begin by considering a fully connected neural network of depth $D$ that maps the input $x_0\in \bbR^p$ to the output $x_D\in \bbR^p$ following the recursion
$x_{n+1} = \tanh(W_n x_n + b_n)$, where $W_n \in \Op$ are the weight matrices and $b_n \in \bbR^p$ are the biases.
We denote $\Phi_\theta(x)$ the output of the network with input $x\in \bbR^p$ and parameters $\theta = (W_1, b_1, \dots, W_D, b_D)$.
\begin{figure}[h!]
\vspace{-1em}
\centering
\includegraphics[width=\linewidth]{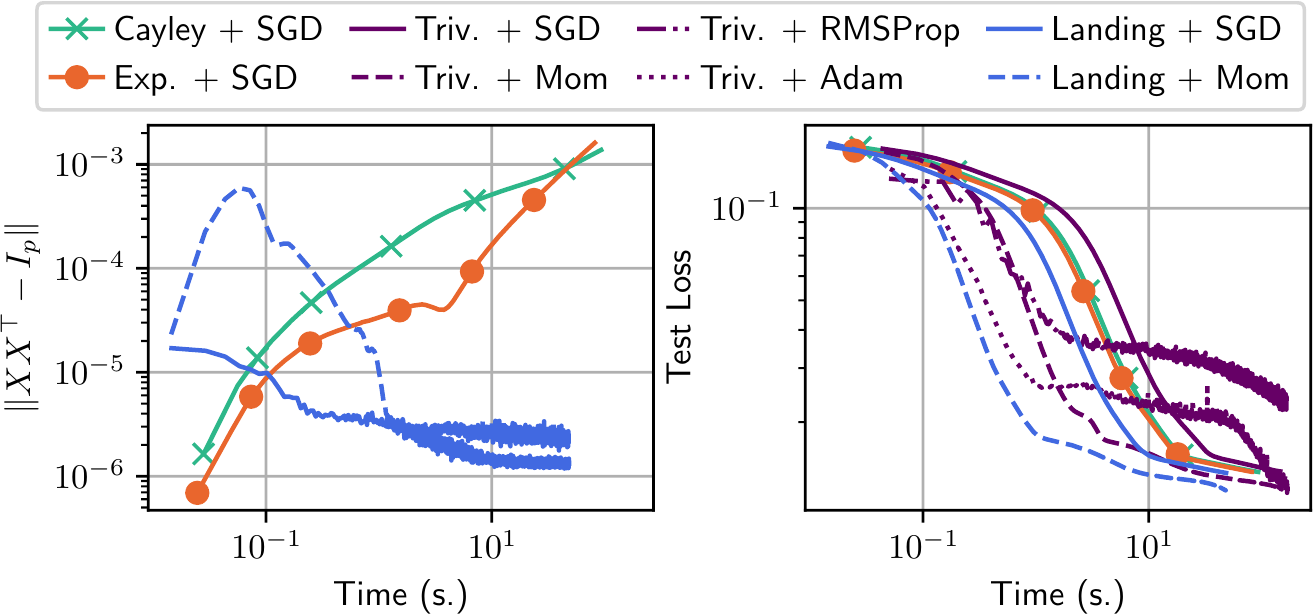}
\caption{Training a fully connected neural network with orthogonal weights. Left: orthogonalization error of all weights. Right: test loss.}
\label{fig:simple_network}
\vspace{-1em}
\end{figure}
In this experiment, we consider a \emph{distillation} problem~\cite{hinton2015distilling}: we generate a  random set of parameters $\theta^*$ that gives the target \emph{teacher} network $\Phi_{\theta^*}$. Then, starting from a random parameter initialization, we try to learn the mapping $\Phi_{\theta^*}$ with a \emph{student} network with parameters $\theta$, by minimizing the loss $\mathcal{L}(\theta)= \sum_{q=1}^Q\|\Phi_{\theta}(x^q) - \Phi_{\theta^*}(x^q)\|^2$, where the $x^q$ are the training examples, drawn i.i.d. from a normal distribution.
We consider the optimization of the orthogonal weights $W_1, \dots, W_L$ with different methods.
We use stochastic Riemannian gradient descent using the exponential or Cayley retraction or the landing flow, with or without momentum for the latter.
We also consider trivializations~\citep{lezcano2019cheap, lezcano2019trivializations}, where each matrix is parametrized as $W_i = \exp(A_i)$ with $A_i\in\Skew_p$, and the optimization is carried over $A_i$ (more details in \autoref{app:triv}). We use trivializations with SGD, SGD + momentum, Adam, and RMSProp.
We use matrices of size $p=100$, with a depth $D=10$. The learning rate is $\eta=0.5$, and the batch size is $256$.
Orthogonalization and test error are displayed in \autoref{fig:simple_network}. Orthogonalization error is not displayed for trivialization methods, since they are exact.
The landing flow with momentum is the fastest method.
It also leads to smaller orthogonalization error than the other retraction methods, because it does not suffer from accumulation of numerical errors (see \autoref{sec:num_errors}).

\begin{figure}[h!]
\centering
\includegraphics[width=.99\linewidth]{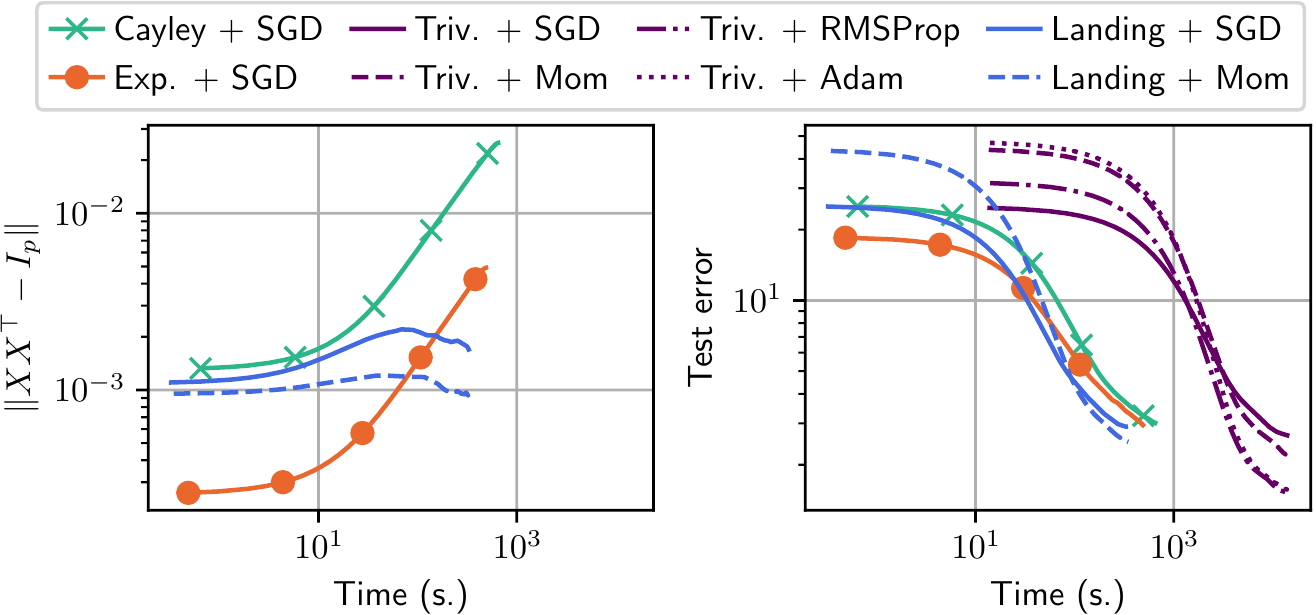}
  \caption{Training a LeNet5 on MNIST. Left: orthogonalization error. Right: Test Error.}
  \label{fig:mnist}
\vspace{-1em}
\end{figure}

\textbf{LeNet on MNIST}
We train a LeNet5~\citep{lecun1998gradient} for classification on the MNIST dataset. The network has 3 convolutional layers, and we impose an orthogonal constraint on the square kernel matrices. We take a batch size of $4$, and for each algorithm, we take the learning rate that gives the fastest convergence in $4^{i}, i=-5\dots 0$.
We compare the same algorithms as before.
\autoref{fig:mnist} displays the results of our experiment.
Here, the landing algorithm is about $50\%$ faster than retraction methods, and reaches a smaller orthogonalization error.
While trivialization methods with advanced optimizers like RMSprop or Adam allow to reach the smallest test error, these methods are an order of magnitude slower than traditional methods in this case. The main reason is that they backprop through a matrix exponential at each iteration, which is costly (more details in Appendix).
\begin{figure}[h!]
  \centering
  \includegraphics[width=.9\linewidth]{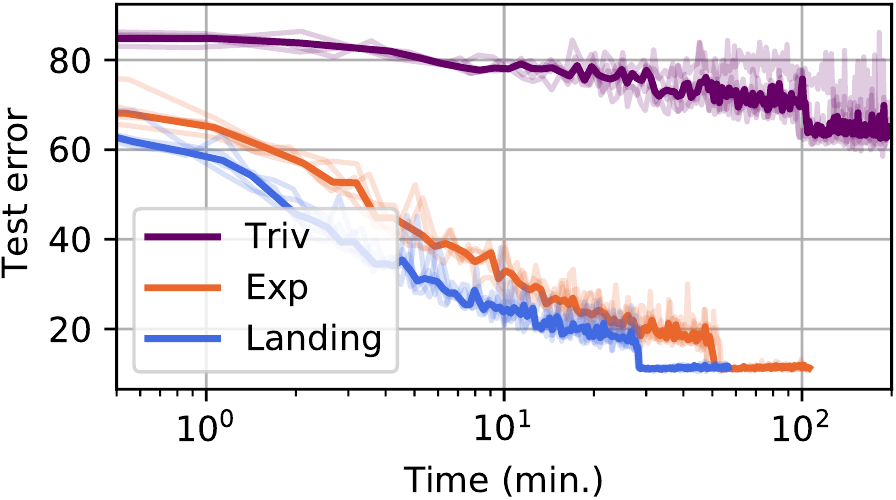}
    \caption{Training a ResNet 18 on CIFAR 10.}
    \label{fig:cifar}
    \vspace{-1em}
  \end{figure}

\newcontent{\textbf{ResNet on CIFAR} In Fig.~\ref{fig:cifar} we train a ResNet18~\cite{he2016deep} on the CIFAR-10 dataset~\cite{krizhevsky2009learning}. Once again, we impose an orthogonality constraint on each convolution kernel. We take a batch size of $128$, and use SGD with momentum to train each algorithm. 
The landing algorithm is here once again $50\%$ faster than retraction-based methods. Also, we notice in this case that trivialization methods fail to reach a high accuracy.}

  \begin{figure}[h!]
    \centering
    \includegraphics[width=.9\linewidth]{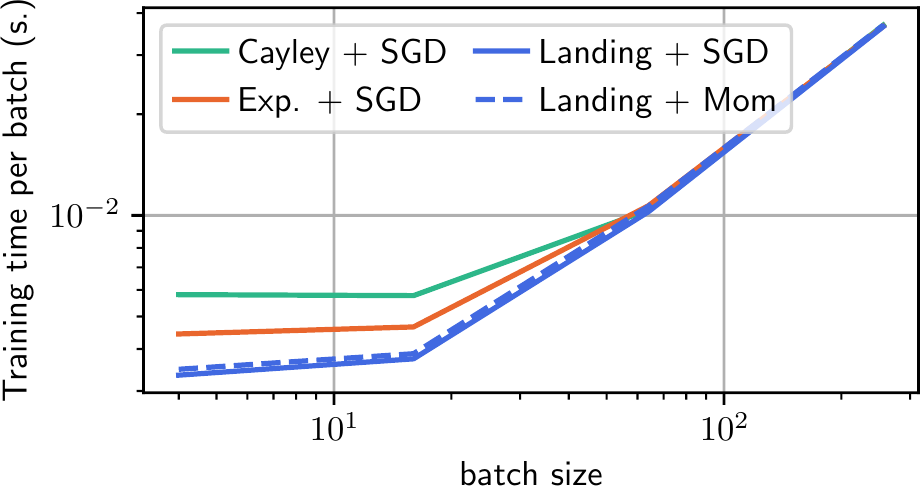}
      \caption{On MNIST with a LeNet5, median training time per batch, as a function of batch size.}
      \label{fig:batch_size}
    \vspace{-1em}
    \end{figure}
\textbf{Diminishing returns} On the MNIST problem, we vary the batch size, and give the training time per batch in Fig~\ref{fig:batch_size}.
The landing method yields the greatest computational benefit when the batch size is small, because in this case, the computational bottleneck is computing the retraction.
As the batch size increases, the computational cost is dominated by backpropagation, and we see a smaller gain with the landing algorithm.
However, the other advantage of the landing algorithm -- that it leads to small orthogonal error -- remains.

\paragraph{Discussion} We have presented a novel method to replace retraction-based algorithms on the orthogonal manifold.
Our method is faster than retractions, it is therefore interesting to use in settings where computing a retraction is the computational bottleneck.
It is also useful when one needs a solution that is accurately orthogonal, since it suffers less from numerical errors than widely used retractions.
Future research directions include the development of second-order methods in this framework, and a thorough extension to the Stiefel manifold.

\bibliography{biblio}

\begin{thebibliography}{10}

\bibitem{ablin2018faster}
P.~Ablin, J.-F. Cardoso, and A.~Gramfort.
\newblock Faster {ICA} under orthogonal constraint.
\newblock In {\em 2018 IEEE International Conference on Acoustics, Speech and
  Signal Processing (ICASSP)}, pages 4464--4468. IEEE, 2018.

\bibitem{absil2007trust}
P.-A. Absil, C.~G. Baker, and K.~A. Gallivan.
\newblock Trust-region methods on {R}iemannian manifolds.
\newblock {\em Foundations of Computational Mathematics}, 7(3):303--330, 2007.

\bibitem{absil2009optimization}
P.-A. Absil, R.~Mahony, and R.~Sepulchre.
\newblock {\em Optimization algorithms on matrix manifolds}.
\newblock Princeton University Press, 2009.

\bibitem{absil2012projection}
P.-A. Absil and J.~Malick.
\newblock Projection-like retractions on matrix manifolds.
\newblock {\em SIAM Journal on Optimization}, 22(1):135--158, 2012.

\bibitem{arjovsky2016unitary}
M.~Arjovsky, A.~Shah, and Y.~Bengio.
\newblock Unitary evolution recurrent neural networks.
\newblock In {\em International Conference on Machine Learning}, pages
  1120--1128, 2016.

\bibitem{balashov2020gradient}
M.~Balashov, B.~Polyak, and A.~Tremba.
\newblock Gradient projection and conditional gradient methods for constrained
  nonconvex minimization.
\newblock {\em Numerical Functional Analysis and Optimization}, 41(7):822--849,
  2020.

\bibitem{bansal2018can}
N.~Bansal, X.~Chen, and Z.~Wang.
\newblock Can we gain more from orthogonality regularizations in training deep
  networks?
\newblock 31:4261--4271, 2018.

\bibitem{bonnabel2013stochastic}
S.~Bonnabel.
\newblock Stochastic gradient descent on {R}iemannian manifolds.
\newblock {\em IEEE Transactions on Automatic Control}, 58(9):2217--2229, 2013.

\bibitem{boumal2020introduction}
N.~Boumal.
\newblock An introduction to optimization on smooth manifolds.
\newblock {\em Available online, May}, 2020.

\bibitem{boyce2017elementary}
W.~E. Boyce, R.~C. DiPrima, and D.~B. Meade.
\newblock {\em Elementary differential equations}.
\newblock John Wiley \& Sons, 2017.

\bibitem{cardoso1996equivariant}
J.-F. Cardoso and B.~H. Laheld.
\newblock Equivariant adaptive source separation.
\newblock {\em IEEE Transactions on signal processing}, 44(12):3017--3030,
  1996.

\bibitem{comon1994independent}
P.~Comon.
\newblock Independent component analysis, a new concept?
\newblock {\em Signal processing}, 36(3):287--314, 1994.

\bibitem{edelman1998geometry}
A.~Edelman, T.~A. Arias, and S.~T. Smith.
\newblock The geometry of algorithms with orthogonality constraints.
\newblock {\em SIAM journal on Matrix Analysis and Applications},
  20(2):303--353, 1998.

\bibitem{gao2019parallelizable}
B.~Gao, X.~Liu, and Y.-x. Yuan.
\newblock Parallelizable algorithms for optimization problems with
  orthogonality constraints.
\newblock {\em SIAM Journal on Scientific Computing}, 41(3):A1949--A1983, 2019.

\bibitem{harris2020array}
C.~R. Harris, K.~J. Millman, S.~J. van~der Walt, R.~Gommers, P.~Virtanen,
  D.~Cournapeau, E.~Wieser, J.~Taylor, S.~Berg, N.~J. Smith, et~al.
\newblock Array programming with numpy.
\newblock {\em Nature}, 585(7825):357--362, 2020.

\bibitem{he2016deep}
K.~He, X.~Zhang, S.~Ren, and J.~Sun.
\newblock Deep residual learning for image recognition.
\newblock In {\em Proceedings of the IEEE conference on computer vision and
  pattern recognition}, pages 770--778, 2016.

\bibitem{helfrich2018orthogonal}
K.~Helfrich, D.~Willmott, and Q.~Ye.
\newblock Orthogonal recurrent neural networks with scaled cayley transform.
\newblock In {\em International Conference on Machine Learning}, pages
  1969--1978. PMLR, 2018.

\bibitem{hinton2015distilling}
G.~Hinton, O.~Vinyals, and J.~Dean.
\newblock Distilling the knowledge in a neural network.
\newblock {\em arXiv preprint arXiv:1503.02531}, 2015.

\bibitem{karimi2016linear}
H.~Karimi, J.~Nutini, and M.~Schmidt.
\newblock Linear convergence of gradient and proximal-gradient methods under
  the polyak-{\l}ojasiewicz condition.
\newblock In {\em Joint European Conference on Machine Learning and Knowledge
  Discovery in Databases}, pages 795--811. Springer, 2016.

\bibitem{kingma2014adam}
D.~P. Kingma and J.~Ba.
\newblock Adam: A method for stochastic optimization.
\newblock {\em arXiv preprint arXiv:1412.6980}, 2014.

\bibitem{krizhevsky2009learning}
A.~Krizhevsky, G.~Hinton, et~al.
\newblock Learning multiple layers of features from tiny images.
\newblock 2009.

\bibitem{lecun1998gradient}
Y.~LeCun, L.~Bottou, Y.~Bengio, and P.~Haffner.
\newblock Gradient-based learning applied to document recognition.
\newblock {\em Proceedings of the IEEE}, 86(11):2278--2324, 1998.

\bibitem{lee2013smooth}
J.~M. Lee.
\newblock Smooth manifolds.
\newblock In {\em Introduction to Smooth Manifolds}, pages 1--31. Springer,
  2013.

\bibitem{lezcano2019trivializations}
M.~Lezcano~Casado.
\newblock Trivializations for gradient-based optimization on manifolds.
\newblock {\em Advances in Neural Information Processing Systems},
  32:9157--9168, 2019.

\bibitem{lezcano2019cheap}
M.~Lezcano-Casado and D.~Mart{\i}nez-Rubio.
\newblock Cheap orthogonal constraints in neural networks: A simple
  parametrization of the orthogonal and unitary group.
\newblock In {\em International Conference on Machine Learning}, pages
  3794--3803. PMLR, 2019.

\bibitem{li2020efficient}
J.~Li, F.~Li, and S.~Todorovic.
\newblock Efficient riemannian optimization on the stiefel manifold via the
  cayley transform.
\newblock In {\em International Conference on Learning Representations}, 2020.

\bibitem{moler2003nineteen}
C.~Moler and C.~Van~Loan.
\newblock Nineteen dubious ways to compute the exponential of a matrix,
  twenty-five years later.
\newblock {\em SIAM review}, 45(1):3--49, 2003.

\bibitem{nishimori1999learning}
Y.~Nishimori.
\newblock Learning algorithm for independent component analysis by geodesic
  flows on orthogonal group.
\newblock In {\em IJCNN'99. International Joint Conference on Neural Networks.
  Proceedings (Cat. No. 99CH36339)}, volume~2, pages 933--938. IEEE, 1999.

\bibitem{oja1982simplified}
E.~Oja.
\newblock Simplified neuron model as a principal component analyzer.
\newblock {\em Journal of mathematical biology}, 15(3):267--273, 1982.

\bibitem{paszke2019pytorch}
A.~Paszke, S.~Gross, F.~Massa, A.~Lerer, J.~Bradbury, G.~Chanan, T.~Killeen,
  Z.~Lin, N.~Gimelshein, L.~Antiga, et~al.
\newblock Pytorch: An imperative style, high-performance deep learning library.
\newblock {\em arXiv preprint arXiv:1912.01703}, 2019.

\bibitem{qi2010riemannian}
C.~Qi, K.~A. Gallivan, and P.-A. Absil.
\newblock {R}iemannian {BFGS} algorithm with applications.
\newblock In {\em Recent advances in optimization and its applications in
  engineering}, pages 183--192. Springer, 2010.

\bibitem{rodriguez2016regularizing}
P.~Rodr{\'\i}guez, J.~Gonzalez, G.~Cucurull, J.~M. Gonfaus, and X.~Roca.
\newblock Regularizing cnns with locally constrained decorrelations.
\newblock {\em arXiv preprint arXiv:1611.01967}, 2016.

\bibitem{schonemann1966generalized}
P.~H. Sch{\"o}nemann.
\newblock A generalized solution of the orthogonal procrustes problem.
\newblock {\em Psychometrika}, 31(1):1--10, 1966.

\bibitem{tripuraneni2018averaging}
N.~Tripuraneni, N.~Flammarion, F.~Bach, and M.~I. Jordan.
\newblock Averaging stochastic gradient descent on {R}iemannian manifolds.
\newblock {\em arXiv preprint arXiv:1802.09128}, 2018.

\bibitem{xiao2020exact}
N.~Xiao, X.~Liu, and Y.~Yuan.
\newblock Exact penalty function for l2, 1 norm minimization over the stiefel
  manifold.
\newblock {\em SIAM J. Optim}, 2020.

\bibitem{xiao2020class}
N.~Xiao, X.~Liu, and Y.-x. Yuan.
\newblock A class of smooth exact penalty function methods for optimization
  problems with orthogonality constraints.
\newblock {\em Optimization Methods and Software}, pages 1--37, 2020.

\bibitem{xiao2021penalty}
N.~Xiao, X.~Liu, and Y.-x. Yuan.
\newblock A penalty-free infeasible approach for a class of nonsmooth
  opimtization problems over the stiefel manifold.
\newblock {\em arXiv preprint arXiv:2103.03514}, 2021.

\bibitem{yan1994global}
W.-Y. Yan, U.~Helmke, and J.~B. Moore.
\newblock Global analysis of {O}ja's flow for neural networks.
\newblock {\em IEEE Transactions on Neural Networks}, 5(5):674--683, 1994.

\bibitem{zhang2016first}
H.~Zhang and S.~Sra.
\newblock First-order methods for geodesically convex optimization.
\newblock In {\em Conference on Learning Theory}, pages 1617--1638, 2016.

\bibitem{zhang2018towards}
H.~Zhang and S.~Sra.
\newblock Towards {R}iemannian accelerated gradient methods.
\newblock {\em arXiv preprint arXiv:1806.02812}, 2018.

\end{thebibliography}
\bibliographystyle{abbrv}

\newpage

\newpage

\appendix

\onecolumn

\section{Proofs}
\subsection{Proof of \autoref{prop:relat_eucl_ders}}
We recall the Euclidean Taylor expansion of $f$, where $\nabla f(X)$ is the gradient of $f$ at $X$ and $H_X$ the Hessian of $f$ at $X$:
\begin{align}
  \label{eq:app:taylor}
  f(X + E) &= f(X) + \langle \nabla f(X), E\rangle + o(\|E\|),\\
  \nabla f(X + E) &= \nabla f(X) + H_X(E) + o (\|E\|).
\end{align}

The relative gradient $\psi(X)$ is such that
$$f(X +AX) = f(X) + \langle \psi(X), A\rangle + o(A).$$
Letting $E = AX$ in~\eqref{eq:app:taylor} gives, on the other hand
$$
f(X+AX) = f(X) + \langle \nabla f(X), AX\rangle + o(A).
$$
Identification of the first order term shows that for all $A\in\Skew_p$, it holds
$$
\langle \psi(X), A\rangle = \langle \nabla f(X), AX\rangle\enspace,
$$
or by transposition:
$$
\langle \psi(X) - \nabla f(X)X^{\top}, A\rangle=0.
$$
This scalar product cancels for all $A\in\Skew_p$, so the matrix on the left has to be in the orthogonal of $\Skew_p$, i.e. it is a symmetric matrix. In other words, its skew-symmetric part cancels. We therefore have
$$
\Skew(\psi(X) - \nabla f(X)X^{\top}) = 0,
$$
and since $\psi(X)$ is skew-symmetric, we find
$$
\psi(X) = \Skew(\nabla f(X)X^{\top}).
$$

For the relative Hessian, we have
\begin{align*}
  \psi(X + AX) &=\Skew\left(\nabla f(X + AX)(X + AX)^{\top}\right)\\
  &=\Skew\left((\nabla f(X) + H_X(AX))(X + AX)^{\top}\right) + o(A) \\
  &= \Skew\left(\nabla f(X)X^{\top}\right) + \Skew\left(H_X(AX)X^{\top}\right) + \Skew\left(\nabla f(X)X^{\top}A^{\top}\right) + o(A)
\end{align*}

By identification of the first order term, we find
$$
\mathcal{H}_X(A) =  \Skew\left(H_X(AX)X^{\top} - \nabla f(X)X^{\top}A\right).
$$

\subsection{Proof of \autoref{prop:relat_riemannian_ders}}
The Riemannian gradient, $\grad f(X)$, and Hessian, $\Hess f(X)$, are such that for $\xi\in\mathcal{T}_x$, it holds
$$
f(\mathcal{R}(X, \xi)) = f(X) + \langle \grad f(X), \xi\rangle+\frac12 \langle \xi, \Hess f(X)[\xi]\rangle + o(\|\xi\|^2),
$$

where $\mathcal{R}(X,\xi)$ is the exponential retraction: $\mathcal{R}(X, \xi) = \exp(\xi X^{\top})X$. We find using the same method as above:

$$
\grad f(X) = \Skew(\nabla f(X)X^{\top})X
$$
$$
\Hess f(X)[\xi] = \Skew\left(H_X(\xi)X^{\top} - \Sym(\nabla f(X)X^{\top}A)\right)X.
$$
This gives the expected identities.

\subsection{Proof of \autoref{prop:no_poly}}

By contradiction, such polynomial must satisfy:
\begin{itemize}
  \item Orthogonality: for all $X\in\Op$ and $A\in\Skew_p$, $P(X, A)P(X, A)^{\top}=I_p$.
  \item Retraction: for all $X\in \Op$ and $A\in\Skew_p$, $P(X, A) = X+ AX + o(A)$.
\end{itemize}
The first equality shows that the polynomial $PP^{\top}$ is a constant polynomial, it is therefore a polynomial of degree $0$. Since the degree of $PP^{\top}$ is greater than the degree of $P$, $P$ must also be a constant polynomial. This contradicts the second equality.

\subsection{Proof of \autoref{prop:stationnary}}
If $\Lambda(X) = 0$, since $X$ is invertible, we have $\psi(X) + \lambda (XX^{\top} - I_p) = 0$.
This is the sum of two matrices, one skew-symmetric, the other symmetric, which is zero. Therefore, both matrices are $0$, and we deduce $\psi(X) = 0$ and $XX^{\top} = I_p$.

Conversely, if $X\in\Op$ and $\psi(X) = 0$, we have $\Lambda(X) =0$.

\subsection{Proof of \autoref{prop:ortho}}

By differentiation, denoting $n(t) = \mathcal{N}(X(t))$, we find
\begin{align}
\dot n(t) &= -\langle \Lambda(X(t)), \nabla \mathcal{N}(X(t))\rangle\\
&= -\langle \psi(X), (XX^{\top} - I)XX^{\top}\rangle - \lambda\|(XX^{\top} - I_p)X\|^2.
\end{align}
The first term cancels, since $\psi$ is skew-symmetric and $(XX^{\top} - I)XX^{\top}$ is symmetric.
Therefore, we obtain
\begin{equation}
\label{eq:app:diff_ortho}
\dot n(t) = -\lambda \|(X(t)X(t)^{\top} - I_p)X(t)\|^2
\end{equation}
Now, we would like to upper-bound this by a quantity involving only $n(t) = \|X(t)X(t)^{\top} - I_p\|^2$.

Dropping the time $(t)$ for now, and letting $\Delta = XX^{\top} - I_p$, we find

\begin{align}
\|(XX^{\top} - I_p)X\|^2 &= \tr(\Delta XX^{\top}\Delta)\\
&=\tr(\Delta (I_p + \Delta)\Delta)\\
\label{eq:app:upper_bound_N}
&= \|\Delta\|^2 + \tr(\Delta^3)
\end{align}
Next, we need to control $\tr(\Delta^3)$. Denoting $\lambda_1, \dots, \lambda_p$ the eigenvalues of $\Delta$, we have $\tr(\Delta^3) = \sum_{i=1}^p \lambda_i^3$. This is lower bounded by $-\sum_{i=1}^p|\lambda_i|^3$, and then using the non-increasing property of $\ell_p$ norms, we have
$$
\left(\sum_{i=1}^p|\lambda_i|^3\right)^{\frac13} \leq \left(\sum_{i=1}^p|\lambda_i|^2\right)^{\frac12},
$$
and gathering all inequalities together, we have:
$$
\tr(\Delta^3) = \sum_{i=1}^p \lambda_i^3\geq -\sum_{i=1}^p|\lambda_i|^3 \geq -  \left(\sum_{i=1}^p|\lambda_i|^2\right)^{\frac32} =- \|\Delta\|^3.
$$

Finally, using $\mathcal{N}(X) = \|\Delta\|^2$, we find that \autoref{eq:app:upper_bound_N} gives the bound

$$
\|(XX^{\top} - I_p)X\|^2\geq \mathcal{N}(X) - \mathcal{N}(X)^{\frac 32}
$$

We therefore obtain the differential inequation in \autoref{eq:app:diff_ortho}:
$$
\dot n(t) \leq \lambda(n(t) - n(t)^{\frac32}).
$$
Dividing by the right hand side, we get

$$
\frac{\dot n(t) }{n(t) - n(t)^{\frac32}}\leq -\lambda.
$$
And by integration, using the fact that $\gamma: x\mapsto \log(x)- 2\log(1-\sqrt{x})$ is a primitive of $x\mapsto \frac{1}{x - x^{\frac32}}$, it holds:

$$
\gamma(n(t)) - \gamma(N_0) \leq -\lambda t
$$
which overall gives the bound
$$
n(t) \leq \gamma^{-1}\left(-\lambda t + \gamma(N_0)\right)
$$
where the inverse of $\gamma$ is $\gamma^{-1}(x)= \frac{1}{(\exp(-\frac x2) + 1)^2}$.
We get:
\begin{align}
  n(t)&\leq \frac{1}{(\exp(\frac{\lambda t - \gamma(N_0)}{2}) + 1)^2}\\
  &\leq \exp(-\lambda t)\frac{1}{(\exp(- \frac{\lambda t}2) +\exp(-\frac{\gamma(N_0)}{2}))}
\end{align}
The fraction on the right is then upper-bounded by $\exp(\frac{\gamma(N_0)}{2}) = \frac{N_0}{\sqrt{1 - N_0^2}}$, which gives the advertised result.

\subsection{Proof of \autoref{prop:safe_rule}}
Let $X\in\bbR^{p\times p}$, and define $\Delta = XX^{\top} - I_p$ and $A = \psi(X)$.
The landing algorithm maps $X$ to $\tilde{X} = (I_p - \eta (A + \lambda\Delta))X$. Defining $\tilde{\Delta} = \tilde{X}\tilde{X}^{\top}- I_p$, we find
\begin{align}
  \tilde{\Delta} &= (I_p- \eta (A + \lambda\Delta))(\Delta + I_p) (I_p- \eta (-A + \lambda\Delta)) - I_p\\
  &= (1 - 2 \eta\lambda)\Delta + (\eta - \eta^2\lambda)[A, \Delta] -(2\eta \lambda - \eta^2\lambda^2)\Delta^2 - \eta^2 A^2 +\eta^2\lambda^2\Delta^3 + \lambda \eta^2 [A, \Delta^2]-\eta^2 A\Delta A,
\end{align}
where $[A,\Delta] = A\Delta - \Delta A$ is the Lie bracket.

Using the sub-multiplicativity of the norm, and the triangular inequality, denoting $a = \|A\|$ and $d=\|\Delta\|$, we get

\begin{align}
\tilde{d} &\leq (1 - 2\eta\lambda)d + 2(\eta - \eta^2\lambda)ad  + (2\eta \lambda - \eta^2\lambda^2) d^2 + \eta^2a^2 + \eta^2\lambda^2 d^3 + 2\lambda\eta^2 ad^2 + \eta^2 a^2d \\
\label{eq:app:ineq_step}
&\leq  (1 - 2\eta\lambda)d + 2 \eta ad  + 2\eta \lambda  d^2 + \eta^2a^2 + \eta^2\lambda^2 d^3 + 2\lambda\eta^2 ad^2 + \eta^2 a^2d
\end{align}

Reordering terms in ascending powers of $\eta$, we find
$\tilde{d}\leq d -\alpha \eta + \beta \eta^2$ with
$$
\alpha = 2\lambda d -2 ad -2 \lambda d^2
$$
$$\beta = a^2 + \lambda^2d^3 + 2\lambda ad^2 + a^2d > 0$$

Therefore, we find that when $\eta \leq\eta^*(\alpha, \beta) =  \frac{\alpha + \sqrt{\alpha^2 + 4\beta(\varepsilon - d)}}{2\beta}$, we have $\tilde{d}\leq \varepsilon$.

\subsection{First order expansion of the landing field}
In order to study the local convergence of the algorithm, we develop the landing field to the first order.
\begin{proposition}[First order expansion of $\Lambda$]
\label{prop:first_order}
At the first order in $(A, S)$, we have $\Lambda((I_p + A + S)X_*) =\mathcal{J}(A, S)X_*$, with $\mathcal{J}(A, S)\triangleq \mathcal{H}_{X_*}(A) +\mathcal{H}_{X_*}^{\Sym}(S) + \lambda S\enspace$, where $\mathcal{H}_{X_*}$ is the Relative Hessian and $\mathcal{H}_{X_*}^{\Sym}$ is a linear operator from $\Sym_p$ to $\Skew_p$.
\end{proposition}
\begin{proof}
We recall that
$$
\Lambda(X) = \left(\psi(X) + \lambda(XX^{\top} - I_p)\right)X
$$
Letting $X =(I_p + S + A)X_*$, we have at the first order

$$
\psi(X) = \mathcal{H}_{X_*}(A) + \mathcal{H}^{Sym}_*(S)
$$
and
$$
XX^{\top} - I_p = S
$$

which overall gives

$$
\Lambda(X) = \mathcal{J}(A, S)X_*
$$
with $\mathcal{J}(A, S) = \mathcal{H}_{X_*}(A) + \mathcal{H}^{Sym}_*(S) + \lambda S$.

\end{proof}
The linear operator $\mathcal{J}$ can be conveniently written in the basis $(\Skew_p, \Sym_p)$ where it is block-diagonal since for all $A\in\Skew_p$, we have $\Sym\left(\mathcal{J}(A, 0)\right)=0$. The operator is written in this basis
$$
\mathcal{J} = \begin{bmatrix}
 \mathcal{H}_{X_*} & \mathcal{H}_{X_*}^{\Sym} \\
 0 & \lambda Id
\end{bmatrix}\enspace.
$$
As a consequence, the eigenvalues of $\mathcal{J}$ are the eigenvalues of $\mathcal{H}_{X_*}$ and $\lambda$: $\Sp(\mathcal{J}) = \Sp(\mathcal{H}_{X_*})\cup \{\lambda\}$.
\subsection{Proof of \autoref{prop:local_landing}}

First, it is easily seen that since $\mathcal{J}$ is invertible, then the system $\dot X =- \Lambda(X)$ is an ``almost linear'' system~\citep[Ch.9.3]{boyce2017elementary}, and the eigenvalues of $\mathcal{J}$ are all non-negative and real, which shows that $X_*$ is asymptotically stable: therefore, there exists $\delta > 0$ such that the flow, initialized from any $X$ such that $\|X - X_*\|\leq \delta$, converges to $X_*$.
Classical manipulations then give us the advertised convergence speed.
\footnote{See for instance corollary 4.23 of ``Chicone, Carmen. Ordinary differential equations with applications. Vol. 34. Springer Science \& Business Media, 2006.''}

\subsection{Proof of \autoref{prop:local_convergence}}
The landing flow with step $\eta = \frac1L$ iterates $X_{k+1} =\Phi(X_k)$ with $\Phi(X) = X - \frac1L\Lambda(X)$.
The Jacobian of this map is $Id - \frac1L\mathcal{J}(\Lambda)(X)$, where $\mathcal{J}(\Lambda)$ is the Jacobian of $\Lambda$. The eigenvalues of this map are the $1 - \frac \mu L$, where $\mu$ spans the eigenvalues of $\mathcal{J}(\Lambda)(X)$.

Since the eigenvalue of $\mathcal{J}(\Lambda)(X)$ are all real positive at $X^*$, there is a neighborhood of $X^*$ such that in that neighborhood, the eigenvalues of $\mathcal{J}(\Lambda)(X)$ are close to real positive: for $\delta> 0$, there is a neighborhood of $X^*$ such that for $\mu$ an eigenvalue of $\mathcal{J}(\Lambda)(X)$, we have $Re(\mu) > \min(\mu_{min}, \lambda)$ and $|Im(\mu)|\leq \delta$.
Further, thanks to the Lipschitz assumption, we have $|\mu|\leq L$.

As a consequence, the eigenvalues of $\Phi$, in this neighborhood, are of modulus squared:

\begin{align}
  |1 - \frac\mu L|^2 &= (1 -\frac{Re(\mu)}{L})^2 + \eta^2Im(\mu)^2 \\
  &\leq (1 -\frac{\min(\mu_{min}, \lambda)}{L} )^2 + \eta^2\delta^2
\end{align}

Hence, the iterative scheme $X_{k+1} = \Phi(X_k)$ converges at the speed $O((1 -\frac{\min(\mu_{min}, \lambda)}{L} )^k)$.

\subsection{Proof of \autoref{prop:global_conv}}
We have
\begin{align}
\left[f(X(t))\right]' &=\langle \Lambda(X), \nabla f(X)\rangle\\
&=- \langle \Skew(\nabla f(X)X^{\top})X, \nabla f(X)\rangle - \lambda \langle (XX^\top - I_p)X, \nabla f(X)\rangle \\
&= - \|\psi(X)\|^2 - \lambda \langle XX^{\top} - I_p, \Sym(\nabla f(X)X^\top)\rangle
\end{align}

Therefore, using the majorization of $ \Sym(\nabla f(X)X^\top)$, the upper bound on $\|XX^\top - I_p\|$ and Cauchy-Schwarz, we find
$$
 \|\psi(X(t))\|^2 \leq -\left[f(X(t))\right]' + \lambda \exp(-\frac\lambda2t) \frac{\sqrt{N_0}}{\sqrt{N_0} + 1} K
$$
Then, by integration between $t=0$ and $T$, it holds
$$
\int_0^T\|\psi(X(t))\|^2dt\leq f(X_0) -f(X(T)) + 2\frac{\sqrt{N_0}}{\sqrt{N_0} + 1} K\leq f(X_0) - f^* + 2\frac{\sqrt{N_0}}{\sqrt{N_0} + 1} K
$$

Finally, we use
$$
\inf_{t\leq T} \|\psi(X(t))\| \leq \left(\frac1T\int_0^T\|\psi(X(t))\|^2dt\right)^{\frac12}
$$
to obtain the advertised result.

\subsection{Proof of \autoref{prop:discrete_global_conv}}

We assume that we follow the safe rule, so that we are close to the manifold. We let
$\alpha>0$ such that for all iterates, $\|X_k(X_k^\top X_k - I_p)\|^2\leq \alpha \mathcal{N}(X_k)$ and such that the Hessian of $\mathcal{N}$ is bounded by $\alpha$. We also let $\beta$ such that $\|X_k(X_k^\top X_k - I_p)\|^2\geq \beta \mathcal{N}(X_k)$.

We will use the following result extensively:

$$
\|\Lambda(X_k)\|^2 = \|\psi(X_k)\|^2 + \lambda\|X_k(X_k^\top X_k - I_p)\|^2\leq  \|\psi(X_k)\|^2 + \alpha \lambda\mathcal{N}(X_k)
$$
Then, we look at the decrease towards the manifold:
\begin{align}
\mathcal{N}(X_{k+1}) &\leq \mathcal{N}(X_k) - \eta \langle \Lambda(X_k), \nabla \mathcal{N}(X_k)\rangle + \alpha \eta^2\|\Lambda(X_k)\|^2\\
&\leq (1 - \eta \beta)\mathcal{N}(X_k) + \alpha \eta^2(\|\psi(X_k)\|^2 + \alpha \lambda\mathcal{N}(X_k))\\
&=(1 - \eta \beta + \eta^2\alpha^2)\mathcal{N}(X_k) + \alpha\eta^2\|\psi(X_k)\|^2\label{app:eq:ineq_norm}
\end{align}

Next, we turn to the study of the decrease.
Letting $L>0$ the Lipschitz constant of $f$, we have
\begin{align}
f(X_{k+1})&\leq f(X_k) - \eta \langle \Lambda (X_k), \nabla f(X_k)\rangle + \frac12\eta^2L \|\Lambda(X_k)\|^2\\
&\leq f(X_k) - \eta (\|\psi(X_k)\|^2 +\lambda \langle X_kX_k^{\top} - I_p, \nabla f(X_k)X_k^{\top}\rangle) + \frac12\eta^2L(\|\psi(X_k)\|^2 + \alpha \lambda\mathcal{N}(X_k))\label{app:eq:ineq_fun}
\end{align}
Isolating the terms in $\psi(X_k)$, for $\eta \leq \frac1L$, we find

$$
\|\psi(X_k)\|^2 \leq \frac2\eta\left(f(X_k) - f(X_{k+1}) -\eta \lambda \langle X_kX_k^{\top} - I_p, \nabla f(X_k)X_k^{\top}\rangle + \frac12\eta^2\lambda\alpha L\mathcal{N}(X_k)\right)
$$

We let $F$ an upper bound of $\frac{f(X_k) - f(X_{k+1})}{\eta}$, and $G$ and upper bound of $\| \nabla f(X_k)X_k^{\top}\|$ (these quantities exist by compacity since $X_k$ belong to a compact set).

Then, using Cauchy-Schwarz
\begin{equation}
  \label{app:eq:dec_psi}
  \|\psi(X_k)\|^2 \leq 2 F +2\lambda \sqrt{\mathcal{N}(X_k)}G + \eta\lambda\alpha L\mathcal{N}(X_k) 
\end{equation}

Pluging this in Eq.~\eqref{app:eq:ineq_norm}, we get the inequality
$$
\mathcal{N}(X_{k+1}) \leq (1 - \eta \beta  + \alpha \eta^2 + \alpha^2 \eta^3\lambda L)\mathcal{N}(X_k) + 2\lambda G\alpha\eta^2 \sqrt{\mathcal{N}(X_k)} + 2\alpha F\eta^2
$$

To conclude, we majorize

$$
\sqrt{\mathcal{N}(X_k)}\leq \frac1{2}\mathcal{N}(X_k) + \frac{1}{2}
$$
to obtain
$$
\mathcal{N}(X_{k+1}) \leq (1 - \eta \beta  + \alpha \eta^2 + \lambda G\alpha\eta^2 + \alpha^2 \eta^3\lambda L)\mathcal{N}(X_k)  + (2\alpha F + \lambda \alpha G)\eta^2
$$
For $\eta$ small enough, we get
\begin{equation}
  \mathcal{N}(X_{k+1}) \leq (1 - \frac12\eta \beta)\mathcal{N}(X_k)  + (2\alpha F + \lambda \alpha G)\eta^2
  \label{app:eq:rec_norm}
\end{equation}
which gives, starting from $\mathcal{N}(X_0) = 0$,
$$\mathcal{N}(X_k) \leq\eta  \frac{2(2\alpha F + \lambda \alpha G)}{\beta}$$
This shows that $\mathcal{N}(X_k)$ is at most $\propto \eta$.
In the following, for short, we let $\gamma =\frac{2(2\alpha F + \lambda \alpha G)}{\beta} $.

We now use once again~\eqref{app:eq:ineq_fun}:
$$
\|\psi(X_k)\|^2 \leq \frac2\eta\left(f(X_k) - f(X_{k+1}) +\eta^{\frac32} \lambda G\sqrt{\gamma} + \frac12\eta^3\lambda\alpha L\gamma\right)
$$
By averaging up to an integer $K$, we find

$$
\frac1K\sum_{k=1}^K\|\psi(X_k)\|^2 \leq \frac2\eta\left(\frac{f(X_0) - f^*}{K} +\eta^{\frac32} \lambda G\sqrt{\gamma} + \frac12\eta^3\lambda\alpha L\gamma\right)
$$
which gives as advertised
$$
\inf_{k\geq 0}\|\psi(X_k)\|^2 \leq 2 \sqrt{\eta}\lambda G\sqrt{\gamma} +\eta^2\lambda\alpha L\gamma\enspace.
$$

\subsection{Proof of \autoref{prop:discrete_diminish}}

We now assume that the step-size $\eta$ depends on the iterate $k$ with $\eta_k \propto \frac1{k^\alpha}$.

The iterates verify the same inequality~\eqref{app:eq:rec_norm}:
$$
\mathcal{N}(X_{k+1}) \leq (1 - \frac12\eta_k \beta)\mathcal{N}(X_k)  + (2\alpha F + \lambda \alpha G)\eta_k^2
$$

When $\eta_k \propto \frac1{k^{\alpha}}$ with $\alpha < 1$, unrolling this inequality gives
$$
\mathcal{N}(X_{k}) = \mathcal{O}(\frac1{k^\alpha}).
$$
This shows that the iterates converge towards the manifold.

Next, we use once again\eqref{app:eq:ineq_fun}:

$$
\eta_k \|\psi(X_k)\|^2\leq 2 \left(f(X_k) - f(X_{k+1}) +\eta_k^{\frac32} \lambda G\sqrt{\gamma} + \frac12\eta_k^3\lambda\alpha L\gamma\right)
$$

Since $\eta_k$ goes to $0$, we have $\eta_k^3 = o(\eta_k^{\frac32})$. Therefore, for $k$ large enough, we have
$$
\eta_k \|\psi(X_k)\|^2\leq 2 \left(f(X_k) - f(X_{k+1}) +2\eta_k^{\frac32} \lambda G\sqrt{\gamma}\right)
$$

Summing these inequalities up to an integer $K$ gives
$$
\sum_{k=1}^K\eta_k \|\psi(X_k)\|^2\leq 2 \left(f(X_0) - f(X_{K}) +2 \lambda G\sqrt{\gamma}\sum_{k=1}^K\eta_k^{\frac32}\right)
$$

We now have two cases.

\textbf{First case: } When $\alpha > \frac23$, then $\sum_{k=1}^K\eta_k^{\frac32}$ is bounded as $K$ increases, hence $\sum_{k=1}^K\eta_k \|\psi(X_k)\|^2$ is bounded. We then have

$$
\inf_{k\leq K}\|\psi(X_k)\|^2 \leq \frac{\sum_{k=1}^K\eta_k \|\psi(X_k)\|^2}{\sum_{k=1}^K\eta_k} = \mathcal{O}(\frac{1}{\sum_{k=1}^K\eta_k}) = \mathcal{O}(\frac{1}{K^{1 -\alpha}}).
$$
\textbf{Second case: }When $\alpha \leq  \frac23$, then $\sum_{k=1}^K\eta_k^{\frac32}$ is of the order of $K^{1 - \frac32\alpha}$, and we find:

$$
\inf_{k\leq K}\|\psi(X_k)\|^2 \leq \frac{\sum_{k=1}^K\eta_k \|\psi(X_k)\|^2}{\sum_{k=1}^K\eta_k}  = \mathcal{O}(\frac{K^{1 - \frac32\alpha}}{K^{1 -\alpha}}) = \mathcal{O}(\frac1{K^{\frac\alpha 2}})
$$

These two results can be compactly rewritten as 
$$
\inf_{k\leq K}\|\psi(X_k)\|^2 = \mathcal{O}(K^{-\min(\frac \alpha 2, 1 -\frac \alpha 2)})
$$

\subsection{Global convergence of the momentum method in the continuous case}

We consider the momentum extension of the landing flow
\begin{align}
  \dot{A}(t) &= -A(t) + \psi(X(t))\\
  \dot{X}(t) &= -(A(t) + \lambda (X(t)X(t)^{\top} - I_p))X(t)
\end{align}

We consider the energy
$$
E(t) = f(X(t)) + \frac12\|A(t)\|^2
$$
and find
\begin{align}
  E'(t) &= \langle \dot X(t), \nabla f(X(t))\rangle + \langle \dot{A}(t), A(t)\rangle\\
  &= -\lambda \langle X(t)X(t)^{\top} - I_p, \nabla f(X(t))X(t)^\top\rangle -\|A(t)\|^2
\end{align}
As a consequence, it holds
$$
\|A(t)\|^2 \leq E'(t) + \sqrt{\mathcal{N}(X(t))}K
$$
where $K$ bounds $\|\nabla f(X(t))X(t)^\top\|$, and $\|A(t)\|^2$ is integrable.

Next, we let $F(t) = f(X(t))$, and $\Delta(t) = -\lambda \langle X(t)X(t)^{\top} - I_p, \nabla f(X(t))X(t)^\top\rangle$. We have

$$
F'(t) = - \langle A(t), \psi(X(t))\rangle+ \Delta(t)
$$

$$
F''(t) = -\|\psi(X(t))\|^2 + \langle A(t), \psi(X(t))\rangle -\langle A(t), \mathcal{H}_{X(t)}(A(t))\rangle +\Delta'(t)
$$

so that
$$
F''(t) + F'(t) = -\|\psi(X(t))\|^2-\langle A(t), \mathcal{H}_{X(t)}(A(t))\rangle + \Delta(t) + \Delta'(t)
$$
and we have
$$
\|\psi(X(t))\|^2 \leq  - F''(t) - F'(t)  + \Delta(t) + \Delta'(t),
$$
so that $\|\psi(X(t))\|^2$ is integrable. This implies $\lim\inf \|\psi(X(t))\|^2 = 0$.

\section{Comparison with the proximal linearized augmented Lagrangian algorithm method}
\label{app:plam_vs_landing}
In~\cite{gao2019parallelizable}, the authors propose the proximal linearized augmented Lagrangian algorithm (PLAM) method. Like the landing algorithm, it is an infeasible method, which in spirit is very close to our method.

Instead of the landing field
$$
\Lambda_{landing}(X) = \Skew(\nabla f(X) X^{\top})X + \lambda (XX^{\top} - I_p)X
$$
the authors consider
$$
\Lambda_{plam}(X) =\nabla f(X) -  \Sym(\nabla f(X) X^{\top})X + \lambda (XX^{\top} - I_p)X
$$

When $X$ is orthogonal, both methods are similar and we have
$$
\text{For}\enspace X\in\mathcal{O}_p,\enspace \Lambda_{landing}(X) = \Lambda_{plam}(X) = \grad f(X)
$$
However, these fields differ when $X\notin \mathcal{O}_p$.
In theory,~\cite{gao2019parallelizable} provides a global and local convergence proof. However, the proof requires that the gradient of the function is not too large.
On the other hand, we prove global convergence  of the landing algorithm under a very mild Lipschitz assumption on $f$.

In the following, we provide some theoretical and practical arguments to argue that $\Lambda_{landing}(X)$ is far more robust to the choice of the hyper-parameter $\lambda$, and that in some settings, using $\Lambda_{plam}(X)$ can lead to highly instable behavior, while the landing algorithm behaves nicely.

The first argument is that we cannot have a proposition similar to Prop.~\ref{prop:stationnary} for the PLAM field.
Indeed, we have
\begin{proposition}
Let $X\in\bbR^{p\times p}$ such that $\nabla f(X) = SX$ with $S$ a symetric matrix. Let $\Delta = XX^{\top} - I_p$. If $S\Delta + \Delta S = 2 \lambda S$, then $\Lambda_{plam}(X) = 0$
\end{proposition}
\begin{proof}
In this case, we have
\begin{align}
\Lambda_{plam}(X) &= SX -  \frac12(SXX^{\top} + XX^{\top}S)X + \lambda (XX^{\top} - I_p)X \\
& = (S - \frac12(S\Delta + \Delta S + 2S) + \lambda \Delta)X\\
&=  (-\frac12(S\Delta + \Delta S) + \lambda \Delta)X\\
&= 0.
\end{align}
\end{proof}

Therefore, if there is a matrix $X$ such that $\nabla f(X) = \lambda X$, then we automatically have $\Lambda_{plam}(X) = 0$: this means that there might be some spurious critical points of $\Lambda_{plam}$. On the other hand, as demonstrated in Prop.~\ref{prop:stationnary}, the landing field does not have such problem.
In practice, to circumvent this problem, PLAM must assume that the function does not have a too large gradient norm~\cite[Lemma 2.5]{gao2019parallelizable}.

Similarly, we cannot have an orthogonalization property like Prop.~\ref{prop:ortho}. It is in fact easy to exhibit problems where the continuous ODE following the PLAM field
$$
\dot  X = - \Lambda_{plam}(X)
$$
explodes in finite time:
\begin{proposition}
We fix $ p=2$. Let $\alpha > 0$ and consider $f(X) = \|\alpha X - B\|^2$ with $B=\begin{bmatrix}1 & 0 \\ 0 & 0\end{bmatrix}$. Assume that the flow $\dot  X = - \Lambda_{plam}(X)$ starts from $X_0 =\begin{bmatrix}1 & 0 \\ 0 & 1+\delta\end{bmatrix} $. Then, $X(t)$ is of the form $\begin{bmatrix}1 & 0 \\ 0 & 1+\delta(t)\end{bmatrix} $ with $\delta(0) = \delta$. If $\alpha^2 = \beta$, then $\delta(t) = \delta$ for all $t$. If $\alpha^2<\beta$, then $\delta(t)$ goes to infinity in a finite time if $\delta(0) > 0$. If $\delta(0)< 0$, $\delta(t)$ goes to $-1$.
If $\alpha^2 > \beta$, then $\delta(t)$ goes to $0$.
\end{proposition}
\begin{proof}
It is easy to see that for this problem, when $X = \begin{bmatrix}1 & 0 \\ 0 & 1+\delta\end{bmatrix}$, it holds
$$
\Lambda_{plam}(X) =\begin{bmatrix}0 & 0 \\ 0 & (\alpha^2 - \beta)(\delta - 1 - (\delta - 1)^3)\end{bmatrix}
$$
This  shows that $X(t)$ is of the form $\begin{bmatrix}1 & 0 \\ 0 & 1+\delta(t)\end{bmatrix}$ for all $t$, with
$$
\delta'(t) = -(\alpha^2 - \beta)(\delta - 1 - (\delta - 1)^3)
$$
The behavior of $\delta(t)$ is then concluded from an elementary study of the previous ODE.
\end{proof}
In particular, in the setting where $\alpha^2> \beta$, the algorithm can start arbitrarily close to the manifold, and still explode in finite time. In conclusion, global convergence of the continuous flow associated with $\Lambda_{plam}$ does not hold for the previous function. Thanks to Prop.\ref{prop:ortho}, such bad behavior cannot happen for the landing flow.

Note that the discrete algorithm also explodes with the previous problem. This is not just a theoretical concern: the following experiment shows that such bad behavior happens in practice.

We conduct the following experiment. For $p=2$, we generate $A, B\in\bbR^{p\times p}$ with normal i.i.d. entries, and apply both the landing algorithm and PLAM to the minimization of $\|AX - B\|^2$ starting from $X_0 = I_p$. We use $\lambda = 1$ for both algorithms, and a very small step size $\eta = 10^{-3}$.
We generate $10$ different such problems. In Fig.~\ref{fig:plam} we display the distance to the manifold $\|XX^{\top} - I_p\|$ for the $10$ trajectories: PLAM always explodes, while the landing algorithm always succeeds.

As a consequence, the landing algorithm seems more robust than PLAM.

\begin{figure}
\centering
\includegraphics[width=.5\columnwidth]{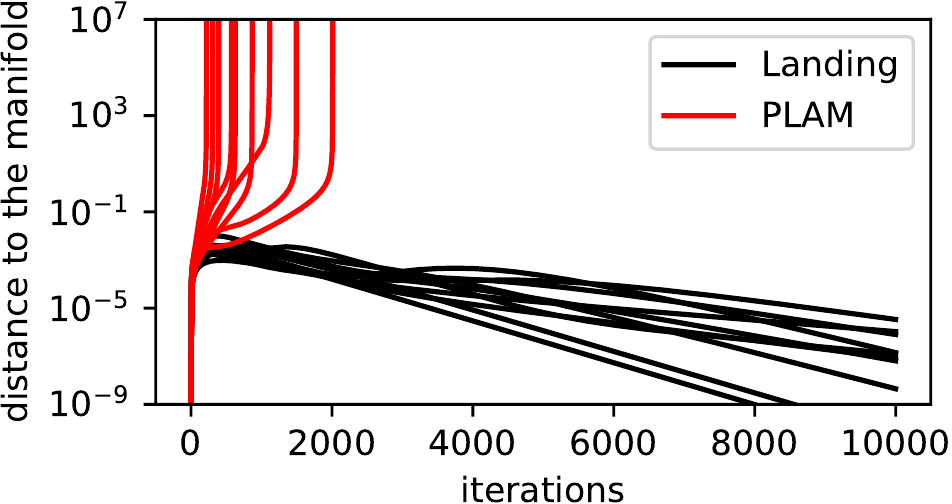}
  \caption{PLAM can explode easily even on simple problems, while the landing algorithm is robust}
  \label{fig:plam}
\end{figure}

\section{Notes on trivializations}
\label{app:triv}
In order to solve the optimization problem $\min_{\mathcal{O}_p}f(X)$, trivializations use a reparametrization of $\mathcal{O}_p$ with a linear space. For instance, we can recast the previous problem as
$$
\min_{A\in\Skew_p}f(\exp(A))
$$

This formulation has several advantages~\cite{lezcano2019cheap}:
\begin{itemize}
\item The optimization problem is now on an Euclidean space, hence is it easy to use Euclidean algorithms like quasi-Newton methods (for instance, L-BFGS), or in the context of deep learning, momentum methods, Adam~\cite{kingma2014adam}, or RMSProp.
\item It is also easy to implement: in a deep learning framework, the layer is parametrized with a simple skew-symmetric matrix, and there is no need to be careful with the optimizer.
\end{itemize}
It also has several drawbacks
\begin{itemize}
\item It may severely change the optimization landscape. For instance, if we consider an orthogonal procrustes problem where $f(X) = \langle M, X\rangle$, the new cost function becomes $\langle M, \exp(A)\rangle$, which has a more complicated structure.To alleviate this problem, it has been proposed to periodically change the foot of the trivialization space during optimization~\cite{lezcano2019trivializations}.
\item It adds a computational cost which can be significant. Indeed, in a deep learning setting, one needs to differentiate through the $\exp$ function. This is done by computing the $\exp$ of a matrix of size $2p\times 2p$. When $p$ is large, this cost may very well be prohibitive, since it is about $8$ times as costly as computing the $\exp$ of a $p\times p$ matrix.
For instance, using the Pytorch implementation of the matrix exponential, on a single laptop CPU, it takes $100$ms to compute the $\exp$ of a $1000 \times 1000$ matrix, while it takes $800$ms to compute the exponential of a $2000 \times 2000$ matrix.
\end{itemize}

\section{Experiments details}
\subsection{Distillation experiment}
The network is a fully-connected multilayer perceptron of depth $L$ that maps the input $x_0$ to the output $x_L$ by the iteration $x_{n+1} = \sigma(W_{n+1}x_n + b_{n+1})$, where $W_n\in\Op$ and $b_n\in\bbR^p$, and $\sigma$ is the $\tanh$ function.
We denote $\Phi_\theta(x)$ the output of the network with input $x\in \bbR^p$ and parameters $\theta = (W_1, b_1, \dots, W_L, b_L)$.
We set $p=100$, and $L=10$. We choose a random set of parameters $\theta^* = (W_1^*, b_1^*, \dots, W_L^*, b_L^*)$ as the teacher network, and starting from a new random initialization, we try to approximate this network.

We let $\theta$ the new parameters, and minimize the loss
$$
\bbE_{x\sim d}[\|\Phi_{\theta}(x) - \Phi_{\theta^*}(x)\|^2]
$$
where the density $d$ is $\mathcal{N}(0, 1)$.

This is done with stochastic gradient descent, using a retraction for orthogonal parameters. For each method, the learning rate the one that yields the fastest convergence in $\{1, 0.1, 0.01, 0.001\}$, and we perform $10000$ iterations with a batch size of $256$.

\subsection{MNIST experiment}
We consider a standard LeNet network which consists of three convolutional layers: the first one maps one channel to $6$ with a kernel of size $5\times 5$, the second maps the $6$ channels to $16$ with a kernel size of $5\times 5$, and the last one maps the $16$ channels to $120$ channels with a kernel size of $4 \times 4$. A last linear layer is used to obtain an output in dimension $10$. The kernel are assumed to be orthogonal. We use a batch size of $4$. For each method, the learning rate the one that yields the fastest convergence in $\{1, 0.1, 0.01, 0.001\}$.

\subsection{CIFAR-10 experiment}

We use a ResNet18 architecture as described in~\cite{he2016deep}. Since there are only $10$ classes in CIFAR-10, we replace the last fully-connected layer so that the output is a vector of size $10$. The rest of the architecture is left unchanged.

We impose orthogonality on all kernels of the network, using the landing method, the exponential retraction, and trivializations. We only compare to the exponential retraction, since in Pytorch it is the fastest retraction.

For all non-orthogonal parameters (biases, and fully connected last layer), we use SGD + momentum and weight decay, with a learning rate of $0.1$, a momentum of $0.9$ and weight decay of $5\times 10^{-4}$.

For the orthogonal parameters, we use SGD + momentum with a learning rate in the grid $\{2, 1, \frac12\}$. Larger learning rates lead to instabilities, and smaller learning rates lead to slow convergence. The momentum is once again taken as $0.9$.

After 100 epochs, the learning rate for all parameters is divided by $10$.

We a batch-size of $128$.

We repeat the traning with $5$ random seeds for each learning rate in the grid. In the figure in the main text, we only display the curves for the learning rate that lead to the fastest convergence. Bold curves correspond to the median of the runs, and individual runs are overlayed with a transparency.

Despite our best effort, we could not make trivializations converge to a satisfying accuracy, even when using a very small or very large learning rate. In this case, trivializations are both much more expensive than the proposed method, and cannot properly train the network.
\end{document}